\documentclass[12pt,a4paper,vi]{sk-thesis}%
\usepackage[utf8]{inputenc}
\usepackage{geometry}
\usepackage{ifdraft}
\usepackage{graphicx}
\setkeys{Gin}{draft=false} 
\usepackage[table]{xcolor}
\usepackage{multicol, multirow}
\usepackage{amsmath, amssymb}
\usepackage{cmap}
\usepackage[T1]{fontenc}
\usepackage[obeyDraft]{todonotes}
\usepackage{booktabs, siunitx}
\usepackage[breaklinks=true]{hyperref}
\usepackage{natbib, breakcites}
\usepackage{bibentry}
\usepackage{doi}
\definecolor{navyblue}{rgb}{0.0, 0.0, 0.5}
\hypersetup{
  final,
  colorlinks=true,
  urlcolor=navyblue,
  citecolor=navyblue,
  linkcolor=navyblue
}
\usepackage{cleveref}
\crefformat{footnote}{#2\footnotemark[#1]#3}
\usepackage{catoptions}
\input{Autoref.tex}
\usepackage{textcomp}
\usepackage{fancyhdr}
\newcommand{\changefont}{%
    \fontsize{10}{10}\selectfont
}
\usepackage{datetime}
\usepackage[automake,toc,acronym]{glossaries}
\glsenablehyper
\makeglossaries
\usepackage{subcaption}
\usepackage{tabularx, lscape, rotating}
\usepackage[most]{tcolorbox}
\newtcolorbox{myframe}[1][]{
  enhanced,
  top=12pt,bottom=12pt,left=16pt,right=16pt,
  arc=0pt,
  outer arc=0pt,
  colback=white,
  boxrule=1pt,
  #1
}

\newcommand{\brackets}[1]{\left(#1\right)}

\newcommand{\eps}{\varepsilon}
\newcommand{\parderiv}[2]{\frac{\partial\,#1}{\partial\,#2}}

\newcommand{\trans}{^\top}
\newcommand{\bx}{\mathbf{x}}
\newcommand{\RR}{\mathbb{R}}
\newcommand{\sign}{\mathrm{sign}\,}

\newcommand{\schatteneps}[2]{\|#1\|_{S_{#2,\varepsilon}}}
\newcommand{\derivative}[2]{\frac{d\,#1}{d\,#2}}

\usepackage{wrapfig}
\usepackage[utf8]{inputenc} 
\usepackage[T1]{fontenc}    
\usepackage{url}            
\usepackage{booktabs}       
\usepackage{amsfonts}       
\usepackage{nicefrac}       
\usepackage{microtype}      
\usepackage{adjustbox}

\usepackage[ruled]{algorithm2e}
\usepackage{color,soul}
\usepackage{amsthm}
\usepackage{tabularx}
\usepackage{verbatim}

\newtheorem{theorem}{Theorem}
\newtheorem{lemma}[theorem]{Lemma}

\theoremstyle{definition}
\newtheorem{definition}{Definition}

\newtheorem{assumption}{Assumption}

\DeclareMathOperator{\tr}{tr}
\DeclareMathOperator{\ent}{\mathcal{S}}
\DeclareMathOperator{\R}{\mathbb{R}}
\DeclareMathOperator{\rank}{\mathrm{rank}\,}
\DeclareMathOperator*{\diag}{diag}
\DeclareMathOperator*{\argmax}{argmax}

\newcommand{\bS}{\mathbb{S}}
\newcommand{\bR}{\mathbb{R}}
\newcommand{\cmmnt}[1]{}

\usepackage{tikz}
\usepackage{pgfplots}
\pgfplotsset{width=11cm,height=7cm}

\usepackage{epigraph}
\ifdraft{
    \includeonly{
        prelude,
        contents,
        chapters/introduction,
        chapters/background,
        chapters/thesis-objectives,
        chapters/conclusion,
        glossary,
        bibliography,
        chapters/appendix-a
    }
}{
}

\begin{document}

\ifdraft{
    \fancypagestyle{plain}{%
      \fancyhf{}%
      \renewcommand{\headrulewidth}{0pt}
      \fancyfoot[L]{\changefont PhD Thesis - Student Name} 
      \fancyfoot[C]{\thepage}
      \fancyfoot[R]{\changefont \textbf{DRAFT} \today} 
    }
    \fancyfoot[L]{\changefont PhD Thesis - Student Name} 
    \fancyfoot[R]{\changefont \textbf{DRAFT} \today} 
    \pagestyle{fancy}
    
    \chapter*{Thesis Title}

\begin{center}
\Large{
\textsc{PhD Thesis (Final Draft) \\ Student Name \\ }
}

\today\ \currenttime
\end{center}

\clearpage

\section*{Abstract}
The central object of this thesis is known under different names in the fields of computer science and statistical mechanics. In computer science, it is called the Maximum Cut problem, one of the famous twenty-one Karp's original NP-hard problems, while the same object from Physics is called the Ising Spin Glass model. This model of a rich structure often appears as a reduction or reformulation of real-world problems from computer science, physics and engineering. However, solving this model exactly (finding the maximal cut or the ground state) is likely to stay an intractable problem (unless $\textit{P} = \textit{NP}$) and requires the development of ad-hoc heuristics for every particular family of instances.

One of the bright and beautiful connections between discrete and continuous optimization is a Semidefinite Programming-based rounding scheme for Maximum Cut. This procedure allows us to find a provably near-optimal solution; moreover, this method is conjectured to be the best possible in polynomial time. In the first two chapters of this thesis, we investigate local non-convex heuristics intended to improve the rounding scheme.

In the last chapter of this thesis, we make one step further and aim to control the solution of the problem we wanted to solve in previous chapters. We formulate a bi-level optimization problem over the Ising model where we want to tweak the interactions as little as possible so that the ground state of the resulting Ising model satisfies the desired criteria. This kind of problem arises in pandemic modeling. We show that when the interactions are non-negative, our bi-level optimization is solvable in polynomial time using convex programming.

\section*{Publications}
\subsection*{Main author}
\nobibliography*
\begin{enumerate}
    \item \bibentry{GMMitigation}
    \item \bibentry{EPSDP}
\end{enumerate}

\subsection*{Co-author}
\begin{enumerate}
    \item \bibentry{Luchnikov}
    \item \bibentry{Pogodin}
\end{enumerate}

\listoftodos
\clearpage
}{
\title{Inference and Optimization for Engineering and Physical Systems}

\author{Mikhail Krechetov}
\department{Skoltech Center for Energy Science and Technology}

\degree{Doctoral Program in Engineering Systems}

\degreemonth{November}
\degreeyear{2021}
\thesisdate{August 2022}


\supervisor{Assistant Professor}{Yury Maximov}


\maketitle



\cleardoublepage
\setcounter{savepage}{\thepage}
\begin{abstractpage}

\end{abstractpage}

\clearpage
\section*{Publications}
\subsection*{Main author}
\nobibliography*
\begin{enumerate}
    \item \bibentry{GMMitigation}
    \item \bibentry{EPSDP}
\end{enumerate}

\subsection*{Co-author}
\begin{enumerate}
    \item \bibentry{Luchnikov}
    \item \bibentry{Pogodin}
\end{enumerate}

\cleardoublepage



    \fancypagestyle{plain}{%
      \fancyhf{}%
      \renewcommand{\headrulewidth}{0pt}
      \fancyfoot[C]{\thepage}%
    }
    \pagestyle{fancy}
}

\newacronym{sk}{SK}{Skolkovo Institute of Science and Technology}
\newacronym{sdp}{SDP}{Semidefinite Programming}
\newacronym{psd}{PSD}{Positive Semidefinite}
\newacronym{lp}{LP}{Linear Programming}
\newacronym{bm}{LR-SDP}{Low-Rank Burer-Monteiro SDP}
\newacronym{epsdp}{EP-SDP}{Entropy-Penalized SDP}
\newacronym{gm}{GM}{Graphical Models}
\newacronym{abm}{ABM}{Agent-Based-Model}
\newacronym{mrf}{MRF}{Markov random fields}
\newacronym{map}{MAP}{Maximum A-Posteriori}
\newacronym{msp}{MSP}{Mixed State Probability}
\newacronym{imp}{IMP}{Ising Model of Pandemic}
\newacronym{icm}{ICM}{Independent Cascade Model}
\newacronym{qcqp}{QCQP}{Quadratically Constrained Quadratic Problems}
\newacronym{svd}{SVD}{Singular Value Decomposition}
\newacronym{mip}{MIP}{Mixed-Integer Problem}
\begin{singlespace}
\printglossary
\end{singlespace}


\begin{singlespace}
    \tableofcontents
\end{singlespace}
\newpage
\begin{singlespace}
    \listoffigures
\end{singlespace}
\newpage
\begin{singlespace}
    \listoftables
\end{singlespace}
\chapter{Introduction}

The thesis is devoted to designing efficient numerical optimization methods for discrete optimization problems focusing mainly on NP-hard quadratic (binary) optimization and its tractable cases. Although being demanded in various applications, such as energy systems, epidemiology, and statistical physics, efficient algorithms for tractable cases of NP-hard optimization problems are mainly an open question. 

The thesis employs the power of semidefinite and conic programming, regularization, rounding, and sampling techniques to design time and memory-efficient numerical algorithms for tractable instances of NP-hard optimization problems. We demonstrate that these techniques allow for faster and more reliable methods in theory and for a wide range of practical problems.


The structure of the thesis is as follows: 
\begin{description}
    \item[\Autoref{cap:background} - Background and Thesis objective] presents overview of thesis-related state-of-the-art results and outlines the main problems resolved in the thesis;

    \item[\Autoref{cap:schatten} - Rank Constraints and Their Relaxations] presents efficient ways of rank constraints relaxation  that proving exactness of the semidefinite programming to some instances of NP-hard problems; 

    \item[\Autoref{cap:epsdp} - Entropy-Penalized Semidefinite Programming] expands the ideas of the previous chapter through the objective regularization; 

    \item[\Autoref{cap:mitigation} - Ising Model Control] applies designed algorithms for modeling COVID-19 pandemic and access the epidemic pattern. 


\end{description}

\chapter{Background and Thesis Objective}
\label{cap:background}

The computational complexity of optimization problems is active research topic for several decades \cite{karp1975computational,gill2019practical}.Despite the significant progress in the theory and practice of tractable optimization settings, where convex optimization is probably the most important~\cite{boyd2004convex}. Similarly, on a number of problems were proved to be computationally hard and unlikely to have polynomial-time algorithms for solving them \cite{arora2009computational}; however, NP-hardness of the problem does not imply absence of its tractable instances. Bounded tree-width instances are probably the most well-known examples of tractable instances of NP-hard problems \cite{cygan2015parameterized}. 

Practical engineering and physics problems, however, can be approximated far beyond their hardness thresholds and often solved exactly. Furthermore, such problem often do not have bounded tree-width and as such do not tractable parameterized algorithms~\cite{cygan2015parameterized}. A notable example of such problem instance are planar~\cite{schraudolph2008efficient} or $K_{3,3}$--free~\cite{likhosherstov2020tractable} Ising models that admit $O(n^{3/2})$ exact solution in polynomial time. 

The thesis proposes an alternative path to approach tractability of certain instances of NP-hard binary optimization problems. The latter approach is based on representing the optimization problem in matrix form with the rank one constraint on the matrix variable and deriving flexible approximations and relaxations the latter constraint. The approach lead to new tractable exact and approximate algorithms for tractable instances of NP-hard optimization problems. These algorithms provide state-of-the-art results for a number of practical problems in statistical physics, energy systems, and epidemic modeling.











\todo[inline]{TODO: complete chapter}
\chapter{Rank Constraints and Their Relaxations}
\label{cap:schatten}

\section{Introduction}

Binary quadratic optimization is a classical combinatorial optimization problem, which finds a wide range of applications in computer vision \cite{wang2017large, ren2014optimizing}, circuit layout design \cite{Barahona,kleinhans1991gordian},   computing ground states of Ising Model \cite{MCising}, as well as a number of combinatorial favors \cite{parker2014discrete,garey2002computers,cormen2009introduction}.  For  comprehensive  list of applications we refer to \cite{Deza, UBQP}.

A special case of this problem is unconstrained binary quadratic programming (UBQP). It is a classical  NP-hard problem, hardly possible to be solved exactly in polynomial time. A number of relaxation techniques substituting the original problem to a convex one has been proposed. Linear, second-order cone and semidefinite relaxations are among them. 

\acrfull{sdp} relaxation  has been shown to lead to tighter approximation than other relaxation methods for many combinatorial optimization problems including  binary quadratic optimization ones (\cite{goemans1995improved} and \cite{Khot:2002:PUG:509907.510017}).  Still, \acrshort{sdp} reduces the problem to convex optimization over the cone of \acrfull{psd} matrices and outputs a full-rank matrix requiring to be rounded  to obtain a vector valued solution. 

In this chapter we address the question of \acrfull{bm}, which is aimed at strengthening results of the standard \acrshort{sdp} relaxation. While a number of methods has been proposed (see \cite{Lemon} for a survey),  we discuss direct rank minimization of an obtained \acrshort{sdp} solution. We use smoothed rank approximations in order to reduce the rank of the \acrshort{sdp} solution without significant loss in the optimal value. For this purpose we propose an efficient first-order optimization procedure which does not require projecting on the feasible set. This will potentially lead to more accurate rounding procedure allowing to obtain a better vector solution.

In the rest of the chapter we discuss problems of the form
\begin{equation}
\begin{split}
\max\ &x\trans Ax,\\
\text{s.t.}\ &x\in\{-1, 1\}^n,
\end{split} 
\label{eq::bin_quadr_problem}
\end{equation}
where $A$ is an arbitrary symmetric $n\times n$ matrix.

There are many well-known problems that can be naturally written in this form: the maximum cut problem, the 0-1 knapsack problem, the linear quadratic regulator and many others. 

Not all of these formulations appear in the form \ref{eq::bin_quadr_problem}. Instead, some problems might have a linear term of the from $b\trans x$ and a 0-1 constraint, that is $x\in \{0,1\}^n$. Nevertheless, such problems might be converted to the form \ref{eq::bin_quadr_problem}, as showed in \cite{Helmberg1995}.

For instance, in the maximum cut (max-cut) problem we want to find a partition of graph's vertices into two disjoint sets such that the sum of edges between these sets is maximal. This problem is NP-hard \cite{Karp1972}, however, it can be solved approximately. For small instances (up to $50$ vertices) the maximum cut may be solved efficiently with the branch-and-bound algorithm \cite{krislock2014improved}, \cite{Rinaldi}. For bigger problems (around $1000$ vertices) the semidefinite relaxation discussed below gives the best known approximations. It is crucial for many applications to improve existing approximations or to extend them on large-scale instances.

Quadratic boolean programming (\ref{eq::bin_quadr_problem}) is a particular case of \acrfull{qcqp}, so general heuristics for this class of problems may be applied. 

\subsection{Semidefinite relaxation}
Standard \acrshort{sdp} relaxation leads to the following matrix problem:
\begin{equation}
\begin{split}
\max\ & \tr(AX),\\ 
\text{s.t.}\ &\diag X = 1_n,\\
&X\trans=X,\ X\succeq 0. 
\end{split}
\label{eq::sdp}
\end{equation}
This problem is convex and thus could be efficiently solved (we will discuss the particular method below). 

To get a binary solution of the initial problem \ref{eq::bin_quadr_problem}, we decompose the solution of the \acrshort{sdp} relaxation $X=V\trans V$ (via Cholesky decomposition), then take a unit vector with uniformly distributed direction $r$. For each column of $V$, which is $v_i$, we take $x_i=\sign v_i\trans r$. If $A\succeq 0$, the mean result of this procedure is not worse than $2/\pi$ of the maximum value \cite{Nesterov}. In a special case when $A$ is a Laplacian of a graph with non-negative weights this bound can be further improved to $\approx 0.878$ \cite{goemans1995improved}. Note that this famous results cannot be improved if the Unique Games Conjecture is true \cite{Khot:2002:PUG:509907.510017}.

However, this relaxation is exact, when $\rank X=1$. Moreover, low-rank solutions lead to a fewer number of possible binary sets in the rounding procedure described above. This idea becomes more clear if one considers half-space classifiers in $\RR^r$ and $n$ points. The maximum number of different labellings is controlled by Sauer's lemma, and grows as $(n+1)^r$. Hence, if we lie in the vicinity of a correct solution, we would get the correct result from the rounding procedure more likely. This motivation brings us to the idea of low-rank semidefinite programming.

\subsection{Related work}

The existence of low-rank solutions of the problem \ref{eq::sdp} is a fundamental fact discussed in \cite{Pataki98} and \cite{Barvinok1995}. From these works we know that for such problems there exist solutions of the rank at most $r$, where $r(r+1)\leq 2n$. 

Knowing about the existence of low-rank solutions, we may discuss popular approaches in low-rank semidefinite programming. This section is mostly based on the book \cite{Lemon}.

Firstly, the existence of low-rank solutions is used in Burer and Monteiro method \cite{BurerMonteiro}. It is based on the factorization $X=VV\trans$, where $V$ is an arbitrary matrix of size $n\times r$. This problem is non-convex, though it requires much less computations and performs well in practice. However, finding the minimum rank solution of multiple runs of the algorithm with different $r$ (one increments $r$ until resulting point satisfies particular conditions) might be computationally ineffective.

Another approach implies relaxation of the equality constraint \cite{So2008}. It can be written as following:
\[
\tr \brackets{AX}=b \Rightarrow \beta b\leq \tr\brackets{AX}\leq \alpha b,
\] 
where $\alpha$ and $\beta$ control the rank of the solution. In our case the problem \ref{eq::sdp} has $n$ equality constraints of the form $\tr\brackets{XE_{ii}}=1$, where $E_{ii}$ is a zero matrix with unit in the position $i,i$. All $n$ constraints together form $\diag X=1$. Though this approach allows to reduce the rank, the resulting solution satisfies our constraints only approximately. 

The next approach implies that we have already got a solution of the problem \ref{eq::sdp}. This allows us to reduce its rank via some kind of rank minimization procedure keeping the value of $\tr\brackets{AX}$ optimal. Note that rank minimization is NP-hard, so this formulation needs to be further relaxed to be efficiently solvable. This approach has been discussed in the literature, but we have not found a comparison of different ways to minimize the rank in application to boolean quadratic programming.

In this chapter we propose an efficient first-order algorithm, which performs rank minimization. It starts from the solution of the \acrshort{sdp} relaxation. This solution is provided by either CVX interior-point algorithm \cite{cvx}, \cite{gb08} or Burer-Monteiro low-rank procedure, implemented in SDPLR \cite{BurerMonteiro}.

\section{Rank Minimization}

In this section we introduce the problem of rank minimization for the \acrshort{sdp} relaxation. After that we describe existing approaches for rank minimization, and discuss their pros and cons.

\subsection{Problem}

Let $X^*$ be a solution of \ref{eq::sdp} and let $SDP$ be its value and $W^*$ be the value of the binary solution obtained by the rounding procedure. Starting from $X^*$, we want to solve the following problem:
\begin{equation}
\begin{split}
\min\ &\rank X\\
\text{s.t.}\ &\diag X = 1_n,\\
&X\trans=X,\ X \succeq 0,\\
&\tr (AX)= SDP,
\end{split}
\label{eq::min_rank_norelax}
\end{equation}

\subsection{Objective function's relaxations}

However, minimizing the rank is NP-hard. In order to solve the problem \ref{eq::min_rank_norelax}, we replace $\rank X$ with a smooth surrogate, usually non-convex. There is a number of rather popular ways to do that, discussed below.

First of all, the so-called trace norm (or nuclear), defined as 
\begin{equation}
\|X\|_*=\sum_i\sigma_i,
\label{eq::trace_norm}
\end{equation}
where $\sigma_i$ is the $i$-th singular value \cite{Lemon}. In our case, every feasible point of the problem has $\tr X=\|X\|_*=n$, hence this relaxation does not make sense.

Next, the so-called log-det heuristic is to replace $\rank X$ with the concave function
\begin{equation}
\log\det\brackets{X+\eps I}.
\label{eq::logdet}
\end{equation}
This heuristic is discussed, for example, in \cite{Lemon}, \cite{Fazel2003}. Though it performs well in practice, it requires an iterative procedure (described in \cite{Fazel2003}) with an \acrshort{sdp} problem on each iteration. This problem allows to use Burer-Monteiro method \cite{BurerMonteiro}, but it still needs several runs of this algorithm (at least one for each iteration), which is compatible with rank increment in the original Burer-Monteiro procedure. 

The next two relaxations are singular value-based, and in the next section we show that, in fact, they allow an efficient first-order procedure, that does not require projections on the semidefinite cone and hence is computationally efficient.

The first one is the non-convex function of the following form
\begin{equation}
\Phi (X,\,\eps )=(1+\eps^q)\tr\brackets{X\trans(XX\trans+\eps I)^{-1}X}.
\label{eq::rank_singular}
\end{equation}
for $q\in\mathbb{Q}\cap\left[0,1\right]$. Its properties are discussed in \cite{Li}. An important fact is that this relaxation is quite close to the rank:
\[
\begin{split}
&\left|\rank X - \Phi(X,\,\eps) \right|\leq \\
&\leq\eps^q\max\left\lbrace\rank X, \sum_{i=1}^{\rank X}\left|\frac{\eps^{1-q}}{\sigma_i^2(X)}-1\right|\right\rbrace.
\end{split}
\]

The second relaxation utilizes so-called smoothed Schatten p-norms. They are defined as following:
\begin{equation}
\schatteneps{X}{p}^p = \sum\limits_{i\geq 1} \left(\sigma_i^2 + \eps\right)^{\frac{p}{2}} = \tr\left(X\trans X+\varepsilon I\right)^{\frac{p}{2}}.
\label{eq::schatten_norm_smoothed}
\end{equation}
With $p\rightarrow 0$ and $\eps=0$ we get the rank function exactly. Note that for $p<1$ this function is non-convex, and for $p=1$ it is identical to the nuclear norm. Applications of this relaxation can be found in \cite{Nie:2012} and \cite{Mohan2012}. Both papers introduce an iteratively reweighted least squares (IRLS) algorithm in order to solve this problem. However, in our case it requires solving of an \acrshort{sdp} problem with quadratic objective function at each iteration. Hence, it cannot be applied to large-scale problems.

\subsection{Constraint relaxation}

If we omit the last constraint in the problem \ref{eq::min_rank_norelax}, which is $\tr\brackets{AX}=SDP$, rank minimization might occasionally converge to a solution of rank one. Moreover, it may converge to a low-rank vicinity of such solution. If the \acrshort{sdp} relaxation is not tight, then this constraint prevents the procedure from such behavior. 

We can also obtain a binary solution after solving the \acrshort{sdp} relaxation. If we denote the objective value at this point $W^*$, then this value would be a natural lower bound on $\tr\brackets{AX}$. It means that all solutions of rank one above this value are actually better, than the one we got.

This motivation allows us to relax the problem further, and solve (along with rank approximations) the following one:
\begin{equation}
\begin{split}
\min\ &\rank X\\
\text{s.t.}\ &\diag X = 1_n,\\
&X\trans=X,\ X \succeq 0,\\
W^*\le\,&\tr (AX)\le SDP.
\end{split}
\label{eq::min_rank}
\end{equation}
Obviously, the binary solution, that gives $W^*$, is also a solution of the last problem. However, the typical procedure for solving \ref{eq::min_rank} would start from the SDP matrix in order to improve the resulting cut. Since the rank of this matrix is not unit in general, we need to optimize the objective function further.  

An important consequence of this relaxation will be clear in the next section as it allows to avoid projecting on the feasible set.

For completeness we emphasize that our approach cannot be generalized to \acrshort{qcqp} problems. To avoid projecting on the set, it relies significantly on the special structure of constraints that occur in the problem \ref{eq::min_rank}.

\section{Solving Rank Relaxation}
In this section we introduce an efficient first-order procedure that solves the problem \ref{eq::min_rank} without projecting on the positive-semidefinite cone. It can be applied to the singular value relaxation \ref{eq::rank_singular} and smoothed Schatten p-norm \ref{eq::schatten_norm_smoothed}.

\subsection{Efficient first-order procedure}

The problem \ref{eq::min_rank} allows a natural reparametrization to a vector problem of dimension $n(n-1)/2$. To do that, we consider the upper triangular part of the matrix:
\begin{equation}
X=\begin{pmatrix}
1 & x_1 & \dots & x_{n-1}\\
x_1 & 1 & x_n & \dots\\
\dots & \dots & \dots &\dots
\end{pmatrix}.
\label{eq::matrix_vectorized}
\end{equation}
In this case for indices $i,j$ we get $X_{i,j}=x_d$, where $d=\sum_{k=1}^{i-1}(n-k)+j-i$. This satisfies two constraints immediately: $X\trans=X$ and $\diag X=1_n$.

Such reparametrization changes the gradient of the matrix function $f(X)$:
\[
\begin{split}
\parderiv{f(X(x))}{x_{(d)}}=\sum_{i,j}\parderiv{f(X)}{x_{i,j}}\parderiv{x_{i,j}}{x_{(d)}}=\parderiv{f(X)}{x_{i',j'}}\parderiv{x_{i',j'}}{x_{(d)}}+\\
+\parderiv{f(X)}{x_{j',i'}}\parderiv{x_{j',i'}}{x_{(d)}}=\parderiv{f(X)}{x_{i',j'}}+\parderiv{f(X)}{x_{j', i'}},
\end{split}
\]
where $i', j'$ relate to the vector of index $d$. 

Obviously, the upper bound $\tr (AX)\le SDP$ is always satisfied. Moreover, violation of the lower one $W^*\le\tr (AX)$ implies that we got the point that could not improve our binary solution, hence we need to stop.

The last constraint is $X\succeq 0$. We show that proper choice of the gradient step results in a feasible point in case of singular value relaxation and Schatten norms.

First of all, the gradient of the singular value relaxation is
\begin{equation}
\parderiv{\Phi(X,\,\eps)}{X}=2\eps(1+\eps^q)\brackets{XX\trans+\eps I}^{-2}X.
\label{eq::rank_singular_two_deriv}
\end{equation}
For Schatten p-norms we have
\begin{equation}
\derivative{\schatteneps{X}{p}^p}{X}=pX\left(X\trans X+\varepsilon I\right)^{\frac{p-2}{2}}.
\label{eq::schatten_smoothed_derivative}
\end{equation}

If $X$ is symmetric and \acrshort{psd}, then from \acrshort{svd} factorization both gradients are symmetric. Thus in vector parametrization we simply need to multiply the gradient by $2$, and then force diagonal elements to be unit.

For further convenience we denote a symmetric matrix with unit diagonal, upper triangular part of which is constructed from the vector $x$, as $X(x)$. Finally, we show the following:

\begin{theorem}
Let $f(x)$ be the vector-parametrized singular value relaxation \ref{eq::rank_singular} of the matrix $X(x)\succeq 0$. Then for $\alpha\leq \frac{\eps}{4(1+\eps^q)}$ we get $X(x-\alpha \nabla f(x))\succeq 0$.
\label{theorem::singval_psd_constraint}
\end{theorem}

\begin{proof}
The gradient step in upper-triangular parametrization is equivalent to the ordinary gradient step (multiplied by $2$), and then substituting diagonal elements to units. We are going to show that the first step results in a \acrshort{psd} matrix and then the second step keeps matrix \acrshort{psd}.

Consider a symmetric \acrshort{psd} point and its \acrshort{svd} decomposition $X=USU\trans$. Hence for a step $\alpha$ and $C=2\eps (1+\eps^q)$ the new point is
\[
\begin{split}
&X-2\alpha C(XX\trans+\eps I)^{-2}X=\\
&=USU\trans-2\alpha C\brackets{US^2U\trans+\eps I}^{-2}USU\trans=\\
&=U\brackets{S-2\alpha C\brackets{S^2+\eps I}^{-2}S}U\trans.
\end{split}
\]

Hence the positive-semidefiniteness of the resulting matrix is equivalent to such characteristic of the expression in brackets. It is a diagonal matrix. 

If $S_{ii}=0$, then the corresponding diagonal elements are obviously zero. Otherwise we need it to be positive:
\[
\begin{split}
&S_{ii}-2\alpha C\brackets{S_{ii}^2+\eps I}^{-2}S_{ii}\geq 0\Rightarrow\\
&\Rightarrow \alpha \leq \frac{\brackets{S_{ii}^2+\eps I}^{2}}{2 C}=\frac{\brackets{S_{ii}^2+\eps I}^{2}}{4\eps(1+\eps^q)}.
\end{split}
\]
Therefore, it is enough to take
\[
\alpha \leq \frac{\eps}{4(1+\eps^q)}.
\]

Now we want to show that substituting diagonal elements with units is equivalent to adding a diagonal matrix with non-negative entries. In this case, a sum of two \acrshort{psd} matrices is \acrshort{psd}.

It is also equivalent to the fact that all diagonal elements of the gradient are non-negative. This is always true for a symmetric and \acrshort{psd} matrix $X=USU\trans$:
\[
\begin{split}
&\brackets{(XX\trans+\eps I)^{-2}X}_{ii}=\brackets{U\brackets{S^2+\eps I}^{-2}SU\trans}_{ii}=\\
&=\sum_{j}U_{ij}U_{ij}\brackets{S_{jj}^2+\eps I}^{-2}S_{jj}=\\
&=\sum_{j}U_{ij}^2\brackets{S_{jj}^2+\eps I}^{-2}S_{jj}\geq 0.
\end{split}
\]
This observation completes the proof.
\end{proof}

The same technique allows us to get a similar result for smoothed Schatten p-norms:
\begin{theorem}
Let $f(x)$ be vector-parametrized smoothed Schatten p-norm of a matrix $X(x)\succeq 0$. Then for $\alpha\leq \frac{1}{2p}\eps^{(2-p)/p}$ we get $X(x-\alpha \nabla f(x))\succeq 0$.
\label{theorem::schatten_psd_constraint}
\end{theorem}
\begin{proof}
The gradient step in the upper-triangular parametrization is equivalent to the ordinary gradient step (multiplied by $2$) with substitution of diagonal elements with units. We are going to show that the first step results in a \acrshort{psd} matrix and the second step keeps matrix \acrshort{psd}.

Consider a symmetric \acrshort{psd} point and its \acrshort{svd} decomposition $X=USU\trans$. Hence for the step $\alpha$ the new point is
\[
\begin{split}
&X-2\alpha pX\brackets{X\trans X+\eps I}^{(p-2)/2}=\\
&=USU\trans-2\alpha pUSU\trans\brackets{US^2U\trans+\eps I}^{(p-2)/2}=\\
&=U\brackets{S-2\alpha pS\brackets{S^2+\eps I}^{(p-2)/2}}U\trans.
\end{split}
\]

Hence positive-semidefiniteness of the resulting matrix is equivalent to such characteristic of the expression in brackets. It is a diagonal matrix. 

If $S_{ii}=0$, then the corresponding diagonal elements are obviously zero. Otherwise we need them to be positive:
\[
\begin{split}
&S_{ii}-2\alpha pS_{ii}\brackets{S_{ii}^2+\eps I}^{(p-2)/2}\geq 0\Rightarrow\\
&\Rightarrow \alpha \leq \frac{\brackets{S_{ii}^2+\eps I}^{(2-p)/2}}{2 p}.
\end{split}
\]
Therefore, it suffices to take
\[
\alpha \leq \frac{\eps^{(2-p)/2}}{2p}.
\]

Now we want to show that substituting diagonal elements with units is equivalent to adding a diagonal matrix with non-negative entries. In this case a sum of two \acrshort{psd} matrix is \acrshort{psd}.

It is also equivalent to the fact that all diagonal elements of the gradient are non-negative. This is always true for a symmetric and \acrshort{psd} matrix $X=USU\trans$:
\[
\begin{split}
&\brackets{X\brackets{X\trans X+\eps I}^{(p-2)/2}}_{ii}\\
&=\brackets{US\brackets{S^2+\eps I}^{(p-2)/2}U\trans}_{ii}=\\
&=\sum_{j}U_{ij}U_{ij}S_{jj}\brackets{S_{jj}^2+\eps I}^{(p-2)/2}=\\
&=\sum_{j}U_{ij}^2S_{jj}\brackets{S_{jj}^2+\eps I}^{(p-2)/2}\geq 0.
\end{split}
\]
This observation completes the proof.
\end{proof}

In practice, the singular value relaxation bound is much better, since the step bound tends to be larger.

\subsection{Algorithm}

The results above allow us to introduce a gradient descent method, summarized in the algorithm \ref{alg::sing_val_grad} (for singular values relaxation). In this algorithm an abstract procedure ChooseStep returns the appropriate step, which is less or equal to $\frac{\eps}{4(1+\eps^q)}$. The second procedure RoundSolution corresponds to the rounding method, described in the introduction.

\begin{algorithm}
 \SetAlgoLined
    \Begin{
        $X_0=X_{SDP}$\;
        $K=\left\lbrace X\left|X=X\trans,\ X\succeq 0,\ W\leq \tr (AX) \leq SDP\right.\right\rbrace$\;
        \For{$n=1$:max\_iter}{
            $\alpha = \textrm{ChooseStep}(X_n)$\;
            $X_{n} = X_{n-1}-4\alpha \eps(1+\eps^q)\brackets{XX\trans+\eps I}^{-2}X$\;
            $\brackets{X_n}_{ii}=1$\;
            \If{$X_n\not\in K$}{
                break\;
            }
            \If{$\|4\alpha \eps(1+\eps^q)\brackets{X_{n}X_{n}\trans+\eps I}^{-2}X_{n}\|_F < \mathrm{tol}$}{
                break\;
            }
        }
        $\bx,W^{**}=\textrm{RoundSolution}(A, X_{n})$\;
    }
    {\bf Return:} $\bx, W^{**}$.\;
 \caption{Gradient descent for singular values relaxation.}
 \label{alg::sing_val_grad}
\end{algorithm}

\section{Computational experiments}
\subsection{Setup}

We have tested both CVX and SDPLR solvers followed by our algorithm on Gset graphs collection \url{https://web.stanford.edu/~yyye/yyye/Gset/} (originally introduced in \cite{Helmberg1997}). We also tested the Biq Mac library \url{http://biqmac.uni-klu.ac.at/biqmaclib.html} (namely on Beasley instances \cite{Beasley1998}). The latter is a collection of $\{0,1\}$ problems, which are converted to $\{-1,1\}$ ones as mentioned in the introduction.

For each solver we first applied it to the problem. Then we computed the maximum cut based on the \num{1e5} rounding operation. The number of rounding operations was chosen to be completely sure that the solution provides better results compared to others. After that we chose $\eps=0.005$ (smaller values led to ill-conditioned gradients), $q=0.8$ for the singular values relaxation (this value is used by the authors of \cite{Li}), $p=0.1$ and $p=0.01$ for Schatten norms. The stopping criteria for the gradient descent were $100$ iterations and Frobenius norm of the gradient (less than \num{1e-5}). After that a new cut was obtained after \num{1e5} rounding operations. Rank tolerance was chosen to be \num{1e-4}. For Schatten norms the step size was at most $\eps$, which is larger than the theoretical value. Nevertheless, for such steps we have not observed any violations of the \acrshort{psd} constraint.

We tested all methods on the first $21$ Gset graphs. These graphs have $800$ nodes and are solvable with CVX. Another $10$ graphs of size 2000 were tested with SDPLR only.

We also tested all methods on the Beasley instances from the Biq Mac library. We chose relatively large $\{0,1\}$ problems with 250 and 500 nodes. They were tested for CVX only.

\begin{table}
\begin{tabularx}{\columnwidth}{|X|XX|XX|XX|XX|}
\toprule
{}
 & \multicolumn{2}{l}{SDP}& \multicolumn{2}{l}{singular} & \multicolumn{2}{l}{Schatten} & \multicolumn{2}{l}{Schatten} \\
  & \multicolumn{2}{l}{ }& \multicolumn{2}{l}{values} & \multicolumn{2}{l}{p=0.1} & \multicolumn{2}{l}{p=0.01} \\
\hline
\# & rank & cut & rank & cut & rank & cut & rank & cut \\
\midrule
1  &        14 &       11466 &            13 &           11448 &            15 &           11451 &           13 &          11459 \\
\textbf{2}  &        15 &       11436 &            13 &           11438 &            13 &           \textbf{11456} &           14 &          11430 \\
\textbf{3}  &        14 &       11446 &            14 &           11445 &           461 &           \textbf{11455} &           26 &          11453 \\
\textbf{4}  &        14 &       11487 &            14 &           11475 &           319 &           \textbf{11511} &           14 &          11497 \\
\textbf{5}  &        13 &       11462 &            12 &           11462 &            18 &           11451 &           12 &          \textbf{11471} \\
6  &        13 &        2026 &            13 &            1989 &           105 &            2013 &           13 &           2012 \\
\textbf{7}  &        12 &        1833 &            12 &            1821 &            12 &            \textbf{1834} &           12 &           1822 \\
\textbf{8}  &        12 &        1834 &            11 &            1833 &            11 &            \textbf{1840} &           12 &           1831 \\
9  &        12 &        1879 &            12 &            1872 &            12 &            1869 &           12 &           1875 \\
10 &        12 &        1841 &            12 &            1829 &            12 &            1818 &           12 &           1820 \\
11 &        22 &         538 &           138 &             538 &           718 &             538 &          273 &            536 \\
\textbf{12} &        13 &         532 &            27 &             \textbf{536} &           190 &             534 &           41 &            534 \\
13 &        11 &         562 &             8 &             560 &            24 &             562 &            8 &            562 \\
\textbf{14} &        14 &        2994 &            13 &            2995 &            22 &            2992 &           14 &           \textbf{2999} \\
\textbf{15} &        16 &        2979 &            99 &            2982 &           173 &            \textbf{2986} &          122 &           2983 \\
\textbf{16} &        16 &        2978 &            14 &            2981 &           195 &            \textbf{2984} &           15 &           \textbf{2984} \\
17 &        16 &        2978 &            13 &            2975 &           152 &            2974 &           15 &           2974 \\
\textbf{18} &        11 &         924 &            10 &             \textbf{930} &            12 &             921 &           11 &            929 \\
\textbf{19} &        10 &         850 &             9 &             846 &             9 &             \textbf{854} &            9 &            851 \\
\textbf{20} &         9 &         888 &             8 &             882 &             8 &             \textbf{889} &            8 &            884 \\
21 &        10 &         868 &            33 &             864 &            52 &             862 &           41 &            863 \\
\textbf{22} &        18 &       13008 &            18 &           13006 &            18 &           13003 &           18 &          \textbf{13025} \\
23 &        20 &       13010 &            52 &           12985 &            78 &           13010 &           55 &          13004 \\
\textbf{24} &        19 &       13000 &            18 &           13004 &            22 &           13005 &           19 &          \textbf{13010} \\
\textbf{25} &        19 &       13006 &            19 &           12988 &           196 &           \textbf{13026} &           19 &          12987 \\
\textbf{26} &        19 &       12971 &           132 &           12985 &           181 &           12969 &          120 &          \textbf{12990} \\
\textbf{27} &        18 &        2988 &            17 &            2988 &            17 &            \textbf{3027} &           18 &           2989 \\
28 &        20 &        2956 &            20 &            2948 &            54 &            2947 &           20 &           2956 \\
\textbf{29} &        18 &        3044 &            17 &            \textbf{3056} &            23 &            3050 &           17 &           3038 \\
\textbf{30} &        17 &        3076 &            17 &            \textbf{3081} &            17 &            3076 &           17 &           3067 \\
\textbf{31} &        19 &        2947 &            18 &            2955 &            18 &            \textbf{2959} &           18 &           2958 \\
\bottomrule
\end{tabularx}

\caption{SDPLR improvement on Gset graphs}
\label{tab::gset_sdplr}
\end{table}

\begin{table}
\begin{tabularx}{\columnwidth}{|X|XX|XX|XX|XX|}
\toprule
{}
 & \multicolumn{2}{l}{SDP}& \multicolumn{2}{l}{singular} & \multicolumn{2}{l}{Schatten} & \multicolumn{2}{l}{Schatten} \\
  & \multicolumn{2}{l}{ }& \multicolumn{2}{l}{values} & \multicolumn{2}{l}{p=0.1} & \multicolumn{2}{l}{p=0.01} \\
\hline
\# & rank & cut & rank & cut & rank & cut & rank & cut \\
\midrule
1  &        13 &       11462 &            13 &           11448 &            15 &           11451 &           13 &          11456 \\
\textbf{2}  &        13 &       11436 &            13 &           11438 &            13 &           \textbf{11456} &           13 &          11433 \\
\textbf{3}  &        14 &       11446 &            14 &           11445 &           502 &           11446 &           28 &          \textbf{11453} \\
\textbf{4}  &        14 &       11487 &            14 &           11471 &           304 &           \textbf{11511} &           14 &          11497 \\
\textbf{5}  &        12 &       11462 &            12 &           11464 &            18 &           11451 &           12 &          \textbf{11471} \\
6  &        13 &        2024 &            13 &            1994 &           108 &            2013 &           13 &           2016 \\
7  &        13 &        1833 &            12 &            1821 &            12 &            1828 &           12 &           1822 \\
\textbf{8}  &        12 &        1835 &            12 &            \textbf{1856} &           109 &            1846 &           13 &           1839 \\
9  &        12 &        1879 &            12 &            1872 &            12 &            1869 &           12 &           1875 \\
10 &        12 &        1841 &            12 &            1825 &            12 &            1820 &           12 &           1836 \\
11 &        10 &         534 &             6 &             534 &             6 &             534 &            7 &            534 \\
\textbf{12} &         9 &         532 &             8 &             534 &            29 &             \textbf{536} &            8 &            \textbf{536} \\
13 &         8 &         562 &             8 &             560 &            76 &             562 &            8 &            560 \\
\textbf{14} &        13 &        2994 &            13 &            2995 &            22 &            2994 &           13 &           \textbf{2999} \\
\textbf{15} &        13 &        2979 &            13 &            2982 &            51 &            \textbf{2987} &           13 &           2981 \\
\textbf{16} &        14 &        2982 &            14 &            2981 &           589 &            2979 &           61 &           \textbf{2986} \\
17 &        13 &        2978 &            13 &            2978 &           439 &            2973 &           24 &           2976 \\
\textbf{18} &        10 &         924 &            10 &             \textbf{930} &            11 &             920 &           10 &            929 \\
\textbf{19} &         9 &         847 &             9 &             846 &             9 &             \textbf{850} &            9 &            846 \\
\textbf{20} &         9 &         882 &             9 &             887 &           222 &             \textbf{888} &           20 &            886 \\
\textbf{21} &         9 &         862 &             9 &             865 &             9 &             \textbf{867} &            9 &            865 \\
\bottomrule
\end{tabularx}

\caption{CVX improvement on Gset graphs}
\label{tab::gset_cvx}
\end{table}

\begin{table}
\begin{tabularx}{\columnwidth}{|X|XX|XX|XX|XX|}
\toprule
{}
 & \multicolumn{2}{l}{SDP}& \multicolumn{2}{l}{singular} & \multicolumn{2}{l}{Schatten} & \multicolumn{2}{l}{Schatten} \\
  & \multicolumn{2}{l}{ }& \multicolumn{2}{l}{values} & \multicolumn{2}{l}{p=0.1} & \multicolumn{2}{l}{p=0.01} \\
\hline
\# & rank & cut & rank & cut & rank & cut & rank & cut \\
\midrule
1  &          6 &        45369 &              6 &            45369 &            60 &           45369 &            173 &            45369 \\
2  &          6 &        44579 &              6 &            44513 &             7 &           44571 &             47 &            44515 \\
3  &          6 &        48857 &              6 &            48857 &            13 &           48833 &             81 &            48857 \\
4  &          7 &        41094 &              7 &            \textbf{41116} &            17 &           41094 &             88 &            \textbf{41116} \\
5  &          5 &        47685 &              5 &            \textbf{47738} &            16 &           47679 &             75 &            47685 \\
6  &          7 &        40519 &              7 &            \textbf{40545} &             9 &           40475 &             70 &            40469 \\
7  &          6 &        46605 &              6 &            46563 &             6 &           46659 &             38 &            \textbf{46671} \\
8  &          7 &        35076 &              7 &            35000 &             8 &           35076 &             75 &            \textbf{35079} \\
9  &          6 &        48454 &              6 &            \textbf{48570} &             9 &           48447 &             59 &            48364 \\
10 &          6 &        39944 &              6 &            \textbf{39990} &            15 &           39974 &             77 &            39944 \\
\bottomrule
\end{tabularx}

\caption{CVX improvement on the Biq Mac graph collection (250 nodes)}
\label{tab::biqmac_250}
\end{table}

\begin{table}
\begin{tabularx}{\columnwidth}{|X|Xl|Xl|Xl|Xl|}
\toprule
{}
 & \multicolumn{2}{l}{SDP}& \multicolumn{2}{l}{singular} & \multicolumn{2}{l}{Schatten} & \multicolumn{2}{l}{Schatten} \\
  & \multicolumn{2}{l}{ }& \multicolumn{2}{l}{values} & \multicolumn{2}{l}{p=0.1} & \multicolumn{2}{l}{p=0.01} \\
\hline
\# & rank & cut & rank & cut & rank & cut & rank & cut \\
\midrule
1  &          9 &       114540 &              9 &           114440 &            37 &          113880 &            246 &           114100 \\
2  &          8 &       127280 &              8 &           127230 &             8 &          127080 &             16 &           \textbf{127370} \\
3  &          9 &       129080 &              9 &           129210 &            17 &          129130 &            148 &           \textbf{129400} \\
4  &          9 &       128000 &              9 &           128040 &            43 &          127940 &            273 &           \textbf{128170} \\
5  &          8 &       123570 &              8 &           \textbf{123860} &             8 &          123430 &             35 &           123480 \\
6  &          8 &       119770 &              8 &           \textbf{120350} &             8 &          119750 &              8 &           119590 \\
7  &          9 &       120210 &              9 &           119920 &            15 &          120100 &            177 &           120070 \\
8  &          9 &       121940 &              9 &           121650 &            12 &          \textbf{121980} &            101 &           121920 \\
9  &          9 &       118700 &              9 &           118570 &            26 &          118360 &            209 &           118420 \\
10 &          7 &       129220 &              7 &           129100 &             7 &          129120 &             26 &           129040 \\
\bottomrule
\end{tabularx}

\caption{CVX improvement on the Biq Mac graph collection (500 nodes)}
\label{tab::biqmac_500}
\end{table}

\subsection{Discussion}

Results are shown in tables \ref{tab::gset_sdplr} for SDPLR and \ref{tab::gset_cvx} for CVX. 

Our approach outperforms both solvers in terms of the maximum cut on approximately half of the graphs (bold numbers). Moreover, for the first $21$ graphs, it outperforms them on almost the same set of graphs. Among tested rank relaxations the best performance was shown by Schatten norm relaxation with $p=0.1$. Probably, $p=0.01$ was a worse choice as it resulted in bigger influence of the smoothing parameter $\eps$ as it is included in the gradient in order of $\eps^{(p-2)/2}$.

The same result is observed for the $\{0,1\}$ problems, but the singular values relaxation showed the best performance. All methods performed well for rank reduction. However, some runs resulted in extremely large ranks. It show, that the resulting point has a lot of relatively small singular values, which are however not thresholded by \num{1e-4} (note that all singular values sum to the graph's size). This behavior might be considered a drawback of the Schatten norms relaxation and might be caused by the minimization of the sum of singular values (compared to the sum of fractions in the singular values relaxation). Note that small singular values have little effect in the rounding procedure, so this drawback does not have great influence on the method's performance and be further applied to practical problems in statistical physics and power systems \cite{mikhalev2020bayesian, lukashevich2021importance, stulov2020learning}.

\section{Conclusion}
We developed an efficient first-order procedure which is aimed at improvement of \acrshort{sdp} relaxation solutions for quadratic binary programming. We relax rank minimization in terms of the objective function and the linear constraint that controls optimality of the initial \acrshort{sdp} objective. Rank function relaxation is performed by either singular value relaxation or Schatten p-norms. The latter approach showed the best performance on Gset graphs, while the singular values relaxation performed best on the Beasley instances.
\chapter{Entropy-Penalized Semidefinite Programming}
\label{cap:epsdp}

\section{Introduction}

\acrfull{sdp} has become a key tool in solving numerous problems across operations research, machine learning, and artificial intelligence. While there are too many applications of \acrshort{sdp} to present even a representative sample, inference in graphical models \cite{ExpFamilies,erdogdu2017inference}, multi-camera computer vision \cite{torr2003solving}, 
and applications of polynomial optimization \cite{Parrilo,lasserre2015introduction}   
in power systems \cite{7024950}
stand out. Under some assumptions \cite{Madani2014}, the rank at the optimum of the \acrshort{sdp} relaxation is bounded from above by the tree-width of a certain hypergraph, plus one. When a rank-one solution is not available, it is often not needed \cite{Marecek2017}, as one should like to construct a stronger \acrshort{sdp} relaxation.

Penalization of the objective is a popular approach for obtaining low-rank solutions, at least in theory  \cite{recht2010guaranteed,Lemon,zhou19dis,Fawzi2019}. 
Notice that without a further penalization, an interior-point method for \acrshort{sdp} provides a solution on the boundary of the feasible set, where \acrshort{sdp} corresponds to the optimum of the highest rank, whenever there are optima of multiple ranks available. 
The use of a penalization provides a counter-balance in this respect. 
In practice, however, the penalties are often ignored, as it is believed that their computation is too demanding for large-scale problems and does not guarantee low-rank solutions in general.

An alternative approach develops numerical optimization methods that seek \emph{a priori} low-rank solutions. 
This approach, widely attributed to \cite{BurerMonteiro}, considers a factorization of a semidefinite matricial variable $X = V\cdot V^\top$ with $V\in \R^{n\times k}$ for increasing $1 \le k \ll n$. In general, the resulting problems are  non-convex.
Early analyses required determinant-based penalty terms \cite{burer2005local}, although no efficient implementations were known.
Under mild assumptions, for large-enough $k$, 
there is a unique optimum over such a factorization even without a penalization and it recovers the optimum of the initial \acrshort{sdp} problem \cite{boumal2016non}.
For smaller values of $k$, it is known that 
the low-rank relaxation achieves $\mathcal O(1/k)$ relative error \cite{Montanari}. Much more elaborate analyses \cite{erdogdu2018convergence} are now available.
Especially when combined with efficient gradient computation, e.g. within low-rank coordinate  descent \cite[e.g.]{Marecek2017}, this approach can tackle sufficiently large instances and is increasingly popular. 

In this chapter, we aim to develop a method combining both approaches, i.e., utilize an efficient low-rank-promoting penalty in the Burer-Monteiro approach. We present efficient first-order numerical algorithms for solving the resulting penalized problem, with (almost)
 linear-time per-iteration complexity. This makes the combined approach applicable to a wide range of practical problems. 

In a case study, we focus on certain combinatorial optimization problems and inference in graphical models. 
We show that despite the non-convexity of the penalized problem, our approach successfully recovers rank-one solutions in practice.
We compare our solutions against  non-penalized \acrshort{sdp}, belief propagation, and state-of-the-art branch-and-bound solvers of~\cite{krislock2014improved}.

\paragraph{Contribution.} Our contributions can be summarized as follows. We show
\begin{enumerate}
\item  convergence properties of optimization methods employing a wide class of penalty functions that promote low-rank solutions; 
\item linear-time algorithms for computing the gradient of these penalty functions; 
\item computational results on the penalized \acrshort{sdp} relaxation of \acrfull{map} estimates in \acrfull{mrf}, which considerably improve upon the results obtained by interior-point methods and randomized rounding.
\end{enumerate}

This allows for both well-performing and easy-to-analyze low-rank methods for \acrshort{sdp}s coming from graphical models, combinatorial optimization, and machine learning. 

\paragraph{Chapter structure.} This chapter is organized as follows. First, we define the conic optimization problem together with a penalized form with a list of suitable penalization functions. Next, we present theoretical guarantees for solution recovery. These extend known results for solution recovery to the penalty case. Then, we consider the \acrshort{map} problem in \acrfull{mrf} and introduce an iterative procedure for it, together with a first-order method for solving a subproblem at each step; we also show how to compute the gradients efficiently.  Finally, we provide computational experiments for different inference problems in \acrshort{mrf}.

\section{Background}\label{sec:setup}

\acrshort{sdp} is the following conic optimization problem:
\begin{align}
\min_{X\in \bS^n_{+}} \ & 
  \sum_{i \in I}
   f_i( \tr XS_i ) 
   \tag{SDP}
   \label{eq:sdp}
   \\ 
  \text{s.t.:\;} & g_j( \tr X C_j ) \leq 0, \qquad j \in J, 
\notag
\end{align}
where 
$W\in \bS^n_{+}$ denotes that the $n \times n$ matrix variable $W$ is symmetric positive semidefinite,
$I$ and $J$ are finite index sets, 
each $f_i$ and $g_j$ are convex functions $\R^{n\times n} \to \R$, and $C_j \in R^{n\times n}$  and $S_i\in \R^{n\times n}$  are constant matrices.

In the context of combinatorial optimization, one may also consider even more powerful methods such as Sum-of-Squares hierarchies of \cite{Parrilo}.
However, even an \acrshort{sdp} relaxation, which is, in fact, the first step of this hierarchy, may be too computationally challenging. It is usually solved by interior-point methods in the dimension that is quadratic in the number of variables and thus becomes intractable even for medium-scale problems (with a few thousand variables). The problem becomes even less scalable for higher orders of the hierarchy since it requires one to solve \acrshort{sdp} with $n^{\Theta (d)}$ variables, where $d$ is the level of the hierarchy.

\subsection{Low-rank Relaxations and Penalized Problem}
First, let us formally define our notion of a penalty function and explain related work on first-order
methods for \acrshort{sdp}.

\begin{assumption}
Eq.~\eqref{eq:sdp} has an optimum solution with rank~$r$.
\end{assumption}

Let us consider the following proxy problem:
\begin{align}
P_{q,\lambda} \doteq & \min_{V\in R^{n\times q}}
  \sum_{i \in I}
   f_i( \tr(V^\top S_iV) )
   +  R_{q,\lambda}(V) 
   \label{eq:prox} 
   \tag{P-SDP}\\
& \;\text{s.t.: } g_j( \tr(V^\top C_j V) ) \leq 0, \; j \in J.
\nonumber
\end{align}
Where $q > r$ and $R_{q,\lambda}(V)$ satisfies the following:

\begin{definition}[Strict penalty function]\label{def:prop}
A function $\mathcal{R}_{q,\lambda}(V): \R^{n \times n} \to \R$ 
is a \emph{penalty function} that promotes low-rank solutions if for some integers $q' \le q$ and a multiplier $\lambda \in \R^+$: 
\begin{align*}
\lim_{\lambda \to \infty} \mathcal{R}_{q,\lambda}(V)
= \begin{cases}
\; \textrm{0} & \; \mbox{ if } \rank(V) < q',\\
\; \infty & \; \mbox{ if } \rank(V) = q.
\end{cases}
\end{align*}
Moreover, if $q' = q$ and the 
$\rank(X) < q$, then $\mathcal{R}_{q,\lambda}(V) = 0, \ \forall \lambda$, $R(V)$ is a \emph{strict penalty function}.
\end{definition}

We use the word {\it penalty} instead of penalty function that promotes low-rank solutions, where there is no risk of confusion. This notion of a penalty is rather wide. When multiplied by $\lambda$, a determinant is a prime example. One may also consider functions of the following quasi-norms.
\begin{enumerate}
\item The nuclear norm: 
\begin{equation}
\|X\|_* = \sum_i\sigma_i,
\end{equation}
where $\sigma_i$ is the $i$-th singular value, cf. \cite{Lemon}. 
The norm is also known as a trace norm, Schatten 1-norm, and Ky Fan norm. 
As shown by \cite{Srebro}, in the method of \cite{BurerMonteiro}, one can benefit from a bi-Frobenius reformulation:
\begin{align*}
\|X\|_{*} & = \min_{\substack{U\in\mathbb{R}^{n\!\times\! d} \\ V\in\mathbb{R}^{n\!\times\! d}:X=UV^{T}}}\|U\|_{F}\|V\|_{F} 
\\ & = \min_{U,V:X=UV^{T}}\frac{\|U\|^{2}_{F}+\|V\|^{2}_{F}}{2}.
\end{align*}
There are also truncated \cite{HuTruncated} and capped variants \cite{SunTruncated}.
\item Schatten-${p}$ quasi-norm for $p > 0$:
\begin{align}
\|X\|_{S_{p}}=\left(\sum^{n}_{i=1}\sigma^{p}_{i}(X)\right)^{1/p},
\end{align}
where $\sigma_{i}(X)$ denotes the $i$-th singular value of $X$.

\item A smoothed variant of Schatten-${p}$ quasi-norm by \cite{Pogodin} defined in the previous chapter, see eq.~\ref{eq::schatten_norm_smoothed}.

\item Tri-trace quasi-norm of \cite{Shang2018}:
\begin{align}
\|X\|_{\textup{Tri-tr}}
=
\min_{X=UV\Upsilon^\top}\|U\|_{*}\|V\|_{*}\|\Upsilon\|_{*},
\end{align}
which is also the Schatten-1/3 quasi-norm.
\item Bi-nuclear (BiN) quasi-norm of \cite{Shang2018}:
\begin{align}
\|X\|_{\textup{BiN}}
=
\min_{X=UV^{T}}\|U\|_{*}\|V\|_{*},
\end{align}
which is also the Schatten-${1/2}$ quasi-norm.
\item Frobenius/nuclear quasi-norm of \cite{Shang2018}:
\begin{align}
\|X\|_{\textup{F/N}} =
\min_{X=UV^{T}}\|U\|_{*}\|V\|_{F},
\end{align}
which is also the Schatten-${2/3}$ quasi-norm.
\end{enumerate}

We also note there has been considerable interest in the analysis of low-rank approaches without penalization, especially in matrix-completion applications. Much of the analysis goes back to the work of \cite{keshavan2010matrix}. For further important contributions, see \cite{arora2012computing}. 

\subsection{Entropy Viewpoint}

One could see the penalty functions introduced above from the entropy-penalization perspective. 
This is useful not only from a methodological standpoint, but also from a computational one. 
To this end, we consider the Tsallis entropy:
\[
\ent^T_\alpha(X) = \frac{1}{1-\alpha} \left( \frac{\tr{X^\alpha}}{(\tr{X})^\alpha} - 1\right). 
\]

The Tsallis entropy is crucial in our study because it generalizes many popular penalties considered earlier. 
The Schatten $p$-norm 
coincides with the Tsallis entropy $\ent^T_p$ over a set of matrices with a fixed trace norm, so that the tri-quasi norm and bi-nuclear norm (2--6) are covered as well. 
The Log-Det function, $- \log \det X$, which is also used in low-rank \acrshort{sdp}, 
is up to an additive constant factor relative (Shannon) entropy taken concerning a unit matrix, 
while Renyi ($\ent^R$) and von Neumann ($\ent^N$) entropies, 
\begin{align*}
& \ent^R_\alpha(X) = \frac{\log \tr (X/\tr X)^\alpha}{1-\alpha} \text{ and }\\
& S_\alpha^N(X) = - \tr(\log(X/\tr X) \cdot X/\tr(X)), 
\end{align*}
respectively, can also be used as penalties to promote a low-rank solution. To the best of our knowledge, neither Renyi, von Neumann, nor Tsallis entropies have been studied in the context of low-rank \acrshort{sdp}. 

\section{Exact Recovery}\label{sec:theory}

Let us now present a unified view of the penalties and their properties:
\begin{lemma}
Any of:
\begin{enumerate}
\item $\lambda \det ( X )$;
\item $\lambda \sigma_{q}(X)$, where $\sigma_{i}(X)$ denotes the $i$-th singular value of $X$;
\item Tsallis, Renyi, and von Neumann entropies defined on the last $n-q+1$ singular values;
\item $\lambda \max \left\{ 0, \frac { \|X\|_{*} }{ \max \{ \sigma_{\min}(X), \sigma_{q}(X) \} } - q \right\},$
\end{enumerate}
is a penalty function that promotes low-rank solutions.
Moreover, penalties 1--3 are strict.
\label{lem1}
\end{lemma}
\begin{proof} \emph{Sketch}.
(1.) The proof is by simple algebra.
(2.) If $\sigma_{q}(X)$ is 0, we know the rank is $q - 1$ or less. Otherwise, for large values of $\lambda$, the value of the penalties goes to infinity, and hence $q' = q$.
(3.) The definition of entropy assumes that $S(0, ..., 0) = 0$, thus all entropies are strict penalty functions by definition. 
(4.) First, consider the case where all non-zero singular values are equal. In that case, $ \|X\|_{*} / \sigma_{\min}(X) = \rank(X)$, and subtracting $q$ results either in a non-positive number when the rank is less than $q$ or a positive number otherwise.
If the singular values are non-equal, $ \|X\|_{*} / \sigma_{\min}(X) $ provides an upper bound on the rank of $X$, which can be improved as suggested. The use of the upper bound results in the value of the penalty tending to infinity for values between $q'$ and $q$ in the large limit of $\lambda$. 
\end{proof}

Crucially, under mild assumptions, any penalty allows for the recovery of the optimum of a feasible instance of \eqref{eq:sdp} from the iterations of an algorithm on the non-convex problem in variable $V \in \R^{n \times r}$, such as in the methods of \cite{lowrankMAP} or \cite{BurerMonteiro}. In contrast to the traditional results of \cite{burer2005local}, \cmmnt{,burer2005local} who consider the $\det$ penalty, we allow for the use of any strict penalty function. 

\begin{theorem}
Assume that we solve the proxy problem \eqref{eq:prox} iteratively and $\mathcal{R}_{q,\lambda}(V)$ is a strict penalty function that promotes low-rank solutions. In each iteration, if $\mathcal{R}_{q,\lambda}(V) \neq 0$, we increase $\lambda $ (e.g., set $\lambda_{t+1} = \gamma \lambda_t$, with $u > 1$ as some fixed parameter). Furthermore, let us assume that the solution we found is denoted by $\tilde V_{q}$ with $\rank(\tilde V_{q})= q' < q$. Let us also denote $\tilde V_{q'} \in \R^{n \times q'}$ some factorization of $\tilde V_{q} \tilde V_{q}^\top$ (such factorization exists because $\rank(\tilde V_{q})= q'$). Also assume that we have an optimal solution of \eqref{eq:sdp}, $X^*$ with a~$\rank(X^*) = r$.

If 
\begin{equation}
V_{q'+r+1} \triangleq [ \tilde V_{q'}, {\bf 0}_{n\times r}, {\bf 0}_{n\times 1}]
\end{equation}
is a local minimum of $P_{q'+r+1,\lambda}$, then 
$(\tilde V_{q'}) \tilde V_{q'}^\top$ is a global solution of \ref{eq:sdp}.
\label{strict}
\end{theorem}

\begin{proof}
Let us define a family of matrices for $\tau \in [0,1]$ as follows:
\begin{align*}
V(\tau) \triangleq [ \sqrt{\tau} \tilde V_{q'}, \sqrt{(1-\tau)} V_* , {\bf 0}_{n\times 1}],
\end{align*}
where $ (V_*)^\top (V_*)$ is some factorization of $X^*$ with $V_* \in \R^{n \times r} $.

Note that $\forall \tau$, we have $\rank(X(\tau)) < r+q'+1$, and hence
$\forall \lambda, \tau:
R_{q'+r+1, \lambda}(V(\tau)) = 0$.
Now, assume the contradiction, that is, $V_{q'+r+1}$ is a local optimum solution but $\tilde V_{q'}$ is not a global solution.

We show that $\forall \tau \in [0,1]$, $V(\tau)$ is a feasible solution. Indeed, for any $j \in J$ we have
\begin{align*}
    & g_j(\tr(V(\tau)^T C_j \tr(V(\tau)) \leq \\
    & \quad \tau g_j(\tr( [ \tilde V_{q'}, {\bf 0}_{n\times r+1}]^\top C_j [ \tilde V_{q'}, {\bf 0}_{n\times r+1}])) + \\
    & \quad(1-\tau) \tr( [  {\bf 0}_{n\times q'}, V_*,{\bf 0}_{n\times 1}]^\top C_j [  {\bf 0}_{n\times q'}, V_*,{\bf 0}_{n\times 1}])) = \\
    & \hspace{20mm} \tau g_j(\tr(  \tilde V_{q'} ^\top C_j   \tilde V_{q'} )) + (1-\tau) \tr(  X^* C_j  )) \leq  0.
\end{align*}

We just showed that for each $V(\tau), \tau \in [0,1]$ is a feasible point. Now, let us compute the objective value at this point. 
For all $ \tau \in [0,1]$, we have
\[
  \sum_{i \in I}  \tr( V(\tau)^\top S_i V(\tau)) \leq  \tau \sum_{i \in I}
    \tr( \tilde V_{q'}^\top S_i \tilde V_{q'} ) 
   +
 (1-\tau)  
 \sum_{i \in I} \tr( X^* S_i  ) <  \sum_{i \in I} \tr( \tilde V_{q'}^\top S_i \tilde V_{q'} )
\]
which is a contradiction under the assumption that 
$\tilde V_q$ is a local optimum.
\end{proof}

\section{MAP Inference}\label{sec:algo}
A pairwise \acrfull{mrf} is defined for an arbitrary graph $G = (V, E)$ with $n$ vertices. We associate a binary variable $x_i\in \{-1, +1\}$ with each vertex $i\in V$. Let $\theta_i: \{\pm 1\}\to \bR$ and  $\theta_{ij}: \{\pm1\}^2 \to \bR$ defined for each vertex and edge of the graph be vertex and pairwise potential, respectively. Thus, \emph{a posteriori} distribution of $x$ follows the Gibbs distribution: 
\[
p(x|\theta) = \frac{1}{Z(\theta)} e^{U(x|\theta)},
\]
with $U(x;\theta) = \sum_{i\in V} \theta_i(x_i) + \sum_{(i,j)\in E} \theta_{ij}(x_i, x_j).$
The \acrfull{map} estimate is then 
\[
\hat{x} = \argmax\limits_{x\in \{-1, 1\}^n} p(x|\theta) = \argmax\limits_{x\in \{-1, 1\}^n} U(x;\theta),
\tag{MAP}
\]
which is its turn an NP-hard binary quadratic optimization problem, 
\[
\hat x = \argmax\limits_{x\in \{-1, 1\}^n} x^\top S x, 
\]
with indefinite matrix $S$. The \acrshort{sdp} relaxation for this problem is given by \cite{goemans1995improved,Nesterov}:
\begin{gather}
 \min_{X \in \bS^+_n} \tr SX, \quad \text{s.t.: } X_{ii} = 1,\label{eq:lin_sdp} 
\end{gather}
which also covers the Ising model in statistical physics and a number of combinatorial optimization problems. We believe that the approach can be extended to a general setup given by Eq.~\eqref{eq:sdp}. 

An \acrfull{epsdp} relaxation of \eqref{eq:lin_sdp}  has the form 
\begin{align}
 \min_{X \in \bS^+_n} & \tr V^\top S V + R_{\lambda}(V), \quad \text{s.t.: } \|V^i\|_2^2 = 1,\label{eq:lin_sdp_rel}\tag{EP-SDP}
\end{align}
where $V^i$ is the $i$-th column of matrix $V \in \bR^{n\times k}$, $X = V V^\top$. 

\subsection{Numerical Method. } To solve Problem~\eqref{eq:lin_sdp_rel}, we use the Augmented Lagrangian method starting from a sufficiently small value of the penalty parameter $\lambda > 0$ and increasing it in geometric progression, $\lambda_{k+1} = \lambda_{k} \gamma$, with $\gamma > 1$, as summarized in Algorithm \ref{algo:01}. 
The efficiency of the method is due to the efficient computability of gradients of Tsallis, Renyi, and von~Neumann entropies:

\begin{algorithm}
 \SetAlgoLined
    \KwData{Quadratic matrix $S$ of the MAP inference problem, staring point $\lambda_0$, $\gamma > 1$, step size policy $\{\eta_k\}_{k\ge 1}$ accuracy parameters $\varepsilon$, $\epsilon$}
    \KwResult{Solution $V_*$ as a local minimum of \eqref{eq:lin_sdp_rel} of unit rank}
    \Begin{
        $V_0 \leftarrow$ random initialization in $\R^{n\times k}$\;
        \While{$\tr (V_t^\top V_t) - \lambda_{\max} V_t > \varepsilon$}{
        Find local minimum of \ref{eq:lin_sdp_rel}\!$(S, \lambda_t)$ starting from $V_{t-1}$, assign it to $V_t$\;
        \While{$\nabla (\tr V^\top S V + R_{\lambda}(V)) \le \epsilon$}{
        $V = V - \eta_k \nabla (\tr V^\top S V + R_{\lambda}(V))/\|\nabla (\tr V^\top S V + R_{\lambda}(V))\|_2$\;
        $V_i \leftarrow V_i/\|V_i\|_2$ for each row $V_i$\;
        }
        $\lambda_{t+1} = \lambda_t\cdot \gamma$\;
        }
    }
    {\bf Return:} first singular vector of $V_t$.\;
 \caption{Entropy-Penalized SDP.}
 \label{algo:01}%
\end{algorithm}

\begin{lemma}
For any matrix $V \in \bR^{n\times k}$ with $k = {\mathcal O}(1)$, let $X(V) = V^\top V$. Then, gradients of $\ent_\alpha^T(X)$, $\ent^R(X)$, and $\ent^N(X)$ can be computed in $\mathcal{O}(n)$ time. Moreover, if the number of non-zero elements in matrix $A$ is $\mathcal{O}(n)$, then the iteration complexity of Algorithm \ref{algo:01} is $\mathcal{O}(n)$. 
\end{lemma}
\begin{proof}
We start our analysis with Tsallis entropy. First, compute the gradient of $\ent^T_\alpha$ in $V$:
\[
\frac{\partial \ent_\alpha^T(X)}{\partial V} = \frac{\alpha}{1-\alpha}\left(\frac{X^{\alpha-1}}{(\tr{X})^{\alpha}} - \frac{\tr{X^\alpha} }{(\tr{X})^{\alpha + 1}} I \right)V.
\] 

Similarly for Renyi, $\ent^R(X)$, and von Neumann, $\ent^N(X)$, entropies we have
\[
\frac{\partial \ent_\alpha^R(X)}{\partial V} = \frac{\alpha}{1-\alpha}\frac{X^{\alpha-1}}{\tr{X^{\alpha}}} \left( I - \frac{X}{\tr{X}} \right) V 
\]
and
\[
\frac{\partial \ent_\alpha^N(X)}{\partial V} = \frac{\tr{X} I - X}{(\tr{X})^2} \left( I + \log{\frac{X}{\tr{X}}} \right)V.
\]
Following \cite{Holmes}, the singular-value decomposition of matrix $V = U_1 D U_2$ with $U_1 \in \R^{n\times n}$, $D\in \R^{n\times k}$, and $U_2 \in \R^{k\times k}$ can be performed in $\mathcal O(\min{(nk^2, n^2k)}) = \mathcal O(nk^2)$ time. 

For any $\alpha > 1$, the product $X^{\alpha-1}\cdot V = U_1 D^{2\alpha - 1} U_2$ can be computed in time $\mathcal O(n)$ together with $\tr{X^\alpha} = \tr{D^{2\alpha}}$ and $\tr{X} = \tr{D^2}$. Thus, for a fixed $k$,  the gradient $\frac{\partial \ent^T(X)}{\partial V}$ computation time is linear in its dimension. (Here, for any $\alpha \in (0, 1)$, we use the identity $\partial \lambda_i = {\bf  v_i}^\top \partial X {\bf v_i}$.) To finish the proof of the statement, it remains to note that matrix-vector multiplication takes $\mathcal{O}(n)$ time for any matrix with $\mathcal{O}(n)$ non-zero entries. 
\end{proof}

\section{Case Study}\label{sec:case_study}

In this section, we compare our penalized algorithm with other conventional approaches to MAP problems. We fix the width of factorization to $k = 10$, since there is no significant gain in practice for larger values of $k$, cf.  \cite{Montanari}. We choose $2\eta_k \beta = 1$, where $\beta$ is the Lipschitz constant of the gradient in $\ell_2$ norm and $\gamma = 3/2$. Parameters $\lambda_0$ and $\gamma$ of Algorithm~\ref{algo:01} are usually chosen by a few iterations of random search. It is usually enough to have about 35 iterations for penalty updates and a few hundred iterations to find a local minimum using Algorithm~\ref{algo:01}. We emphasize that matrices we obtain by solving \ref{eq:lin_sdp_rel} are rank-one solutions on all \acrshort{map} instances presented. Thus, we do not need any further rounding procedure.

First, in Table~1, we show the performance of our algorithm on selected hard \acrshort{map} inference problems from the BiqMac collection\footnote{http://biqmac.uni-klu.ac.at/biqmaclib.html}. We selected a few of the hardest instances ("gkaif" among them)---dense quadratic binary problems of 500 variables. 

\begin{table}[!t]
\label{tbl:new}
\centering
\begin{adjustbox}{width=\columnwidth,center}
\begin{tabular}{l|c|c|c|c|c}
Instance & gka1f  & gka2f & gka3f & gka4f & gka5f\\
\hline\hline
\multicolumn{6}{c}{SDP}\\
\hline\hline 
objective &  59426 &  97809   &   1347603 & 168616 & 185090  \\
upper bound &66783 & 109826 & 152758 & out of & out of \\
time [s] & 669 & 673 &592 & memory & memory\\
\hline\hline
\multicolumn{6}{c}{EP-SDP}\\
\hline\hline
objective & 60840 &  {\bf 99268}   &  {\bf 136567} & {\bf 170669} & {\bf 189762}  \\
upper bound & n/a & n/a & n/a & n/a & n/a \\
time [s] &   3.3 &  5.0  &  5.3 & 5.2 & 5.7  \\ 
\hline\hline
\multicolumn{6}{c}{Gurobi}\\
\hline\hline
objective &  {\bf 64678} &  97594  & 131898 & 162875 & 189324  \\
upper bound &   73267 &  112223  &  153726 & 190073 & 218428 \\ 
time [s]. &   70 &  70  & 71 & 70 & 70  \\
\end{tabular}
\end{adjustbox}
\caption[Comparison between SDP, EP-SDP and the MIP solver]{The comparison between the standard \acrshort{sdp} relaxation, \acrfull{epsdp} and the \acrshort{mip} solver (Gurobi) for the MAP instances from BiqMac collection.}
\end{table}

We compared our algorithm (\acrshort{epsdp} with Tsallis entropy and $\alpha=2$) with the plain-vanilla semidefinite programming instance solved by the interior-point method, possibly with rounding using the best of one thousand roundings of \cite{goemans1995improved} and also with Gurobi \cite{gurobi}, a \acrfull{mip} solver. To avoid any confusion, we solve the corresponding maximization problems; by the objective value, we mean the value at a feasible solution produced by the method (e.g., rounded solution of \acrshort{sdp} relaxation), which is a lower bound for the corresponding problem.
Because these problems are of the same size (but varying density), the running time of each method is almost constant. It took around 10 minutes for \cite{cvxpy} to solve the \acrshort{sdp} relaxation, and it runs out of memory for the two problems with higher density. Within five seconds, \acrshort{epsdp} obtains results that are better than what Gurobi can produce in 70 seconds. 

\begin{table}[!th]
\centering
\begin{tabular}{l|c|c|c|c}
& \multicolumn{4}{c}{GSET Instance}\\
\hline\hline
& 1 & 2 & 3& 4 \\
\hline
EP-SDP\\(T, $\alpha=2.0$)\!\!& 11485 & 11469 & 11429 & 11442  \\ 
(T, $\alpha=1.1$)\!\!& 11454 & 11463 & 11444 &  11508  \\ 
(R, $\alpha=5$)\!& 11508 & {\bf 11519} & 11496 & {\bf 11531}   \\ 
(R, $\alpha=10$) & {\bf 11520} & 11420 & 11523 & 11523    \\
SDP & 11372 & 11363 & 11279 & 11355 \\ 
Loopy BP & 10210 & 10687 & 10415 & 10389 \\
Mean-Field & 11493 & 11515 & {\bf 11525} & 11512\\
\hline\hline
& \multicolumn{4}{c}{GSET Instance}\\
\hline\hline
& 5 & 6 & 7& 8 \\
\hline
EP-SDP\\(T, $\alpha=2$)& 11427 & 2059 & 1888 & 1866\\ 
(T, $\alpha=1.1$)  & 11506 & 2075 & 1858 &  1895\\ 
(R, $\alpha=5$)  & 11527 & {\bf 2127} & {\bf 1942}  & 1954\\ 
(R, $\alpha=10$) & {\bf 11538} & 2112 & 1940  & {\bf 1958}\\
SDP  & 11313 &  1945 &  1728  & 1727\\ 
Loopy BP & 10143 & 1076 & 964 &  731\\
Mean-Field  & 11528 & 2096 & 1906 & 1912\\
\hline\hline
\end{tabular}
\begin{tabular}{l|c|c|c|c|c}
& \multicolumn{5}{c}{GSET Instance}\\
\hline\hline
& 9 & 10& 11& 12 & 13 \\
\hline
EP-SDP\\(T, $\alpha=2$) & 1933 & 1882 & 532 &  530 & 560 \\ 
(T, $\alpha=1.1$)   & 1969 & 1861 & 544 &  536 & {\bf 568} \\ 
(R, $\alpha=5$)  & 1992 & 1960 & {\bf 550} & {\bf 548} & {\bf 568} \\
(R, $\alpha=10$)  & {\bf 2006} & {\bf 1982}  & 544 & 546 & 564 \\
SDP  & 1767 & 1784 & 524  & 514  &  540 \\ 
Loopy BP & 1021 & 820 & 424 & 412 & 482 \\ 
Mean-Field  & 1940 & 1902 & 542 & 538 & 564\\
\end{tabular}
\label{tbl:02}
\caption[The comparison results between and other methods on MAP instances]{The comparison results between \acrfull{epsdp} and other heuristics for solving \acrshort{map} such as \acrshort{sdp}, Loopy Belief Propagation and Mean-field method. Under the umbrella term \acrfull{epsdp} we evaluate here Tsallis and Renyi entropy regularisers for different values of the parameter $\alpha$.}
\end{table}

In  our study, parameter $\alpha$ of entropies $\mathcal{S}^T_\alpha$, $\mathcal{S}^N_\alpha$, and $\mathcal{S}^R_\alpha$ is chosen on an exponential grid from 1 to 10 with a step 1.1. After experimentation, we note that $\alpha = 1.1$ and $\alpha = 5.0$ seem to improve the results the best for the \acrshort{bm} with Tsallis and Renyi entropies, respectively, although the difference between different $\alpha \in (1,10)$ is not very significant for either of the (Tsallis and Renyi) entropies. 

Table~2 summarizes the results of solving the Max-Cut problem over a GSET collection of sparse graphs\footnote{https://sparse.tamu.edu/Gset}. As we see from the experiments, the results of applying suitable entropy often  outperform both the plain-vanilla \acrshort{sdp} with the classical Goermans-Williamson rounding, the mean-field approximation, as well as the results of UGM solver\footnote{https://www.cs.ubc.ca/~schmidtm/Software/UGM.html} for loopy belief propagation and mean-field inference. 
It is worth noting that for several instances of the GSET graph collection, loopy belief propagation provides rather weak results. 
Usually, strong results of the loopy belief propagation are complementary to those of the mean-field approximation, which is supported by our empirical results. Results of both loopy belief propagation and mean-field approximation can be substantially improved using the linear-programming belief-propagation approach (LP-BP). 

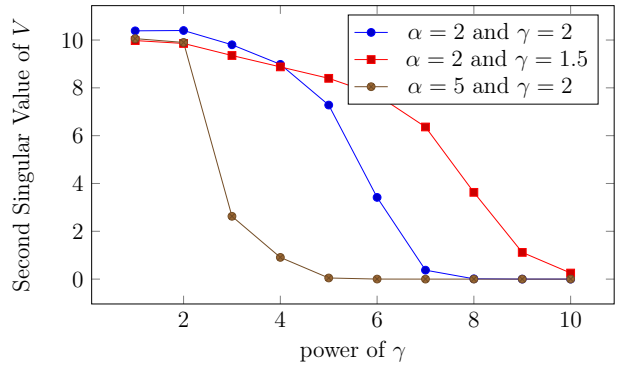
\begin{figure}[t!]
\centering
      \begin{tikzpicture}[scale=0.73]
    \begin{axis}[legend pos=north east, title={},xlabel=power of $\gamma$, ylabel=Second Singular Value of $V$]
         \addplot table [x=iter, y=lambda2-g2, col sep=comma] {tables/epsdp/decreasing_rank.txt};
	  \addlegendentry{$\alpha=2$ and $\gamma = 2$}
         \addplot table [x=iter, y=g1.5, col sep=comma] {tables/epsdp/decreasing_rank.txt};
         \addlegendentry{\hspace{2.5mm}$\alpha=2$ and $\gamma = 1.5$}
         \addplot table [x=iter, y=alpha5-g2, col sep=comma] {tables/epsdp/decreasing_rank.txt};
         \addlegendentry{$\alpha=5$ and $\gamma = 2$}
    \end{axis}
    \end{tikzpicture}
\caption{Rank decrement.}\label{rankdef}
\end{figure}

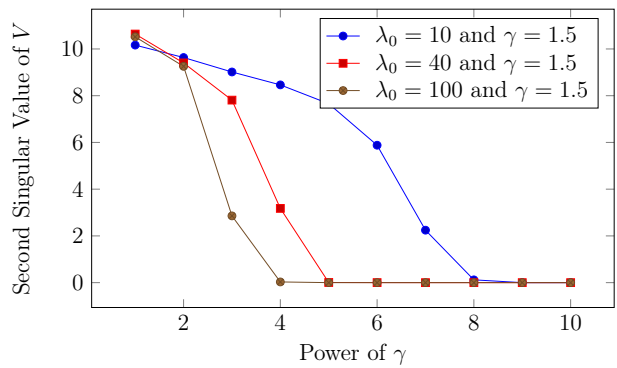
\begin{figure}[t!]
\centering
      \begin{tikzpicture}[scale=0.73]
    \begin{axis}[legend pos=north east, title={},xlabel=Power of $\gamma$, ylabel=Second Singular Value of $V$]
         \addplot table [x=iter, y=10, col sep=comma] {tables/epsdp/inc_lambda.txt};
	  \addlegendentry{$\lambda_0=10$ and $\gamma = 1.5$}
         \addplot table [x=iter, y=40, col sep=comma] {tables/epsdp/inc_lambda.txt};
         \addlegendentry{$\lambda_0=40$ and $\gamma = 1.5$}
         \addplot table [x=iter, y=100, col sep=comma] {tables/epsdp/inc_lambda.txt};
         \addlegendentry{\hspace{2mm}$\lambda_0=100$ and $\gamma = 1.5$}
    \end{axis}
    \end{tikzpicture}
\caption{Rank decrement with multistart.}\label{rankdef2}
\end{figure}

We also want to point out that our iterative algorithm successfully decreases the rank of the solution. The higher the penalization parameter, the lower the rank. We illustrate this in Figure~\ref{rankdef}, where for Tsallis entropies with $\alpha = 2$ and $\alpha = 5$, we plot the second singular value of  matrix $V$. For this plot, we considered the Max-Cut problem for the first graph from the GSET collection and the Tsallis entropy as the penalization function. In Figure~\ref{rankdef2}, we illustrate the same concept for fixed penalization (Tsallis entropy with $\alpha = 2$) and different initial values of the multiplier $\lambda$. We observe that for different penalization functions and update schemes, the rank of the solution decreases gradually with each step. In practice, our iterative algorithm could be seen as a universal rounding procedure for \acrshort{sdp} relaxations. Indeed, if we choose a large-enough penalization update (e.g., $\gamma = 2$ as in Figure 1), we easily obtain a rank-one solution that is not worse and often is substantially better than solutions obtained by randomized rounding. 

Similar problems to the one considered in the chapter appear in a variety of applications such as recommender systems \cite{sidana2021user,anikin2020efficient}, machine learning~\cite{amini2022self,chertkov2020gauges,maximov2018rademacher}, and statistical physics \cite{likhosherstov2019new,likhosherstov2019inference}. 

Overall, we would like to stress that Algorithm~\ref{algo:01} is very fast. This is shown in Figure~3 and Table~3, where we compare run times of \acrshort{epsdp}, \acrfull{bm}, and interior-point method solvers (\acrshort{sdp}) for various Erdos-Renyi random graphs. From the data, we see that (assuming the fixed width of factorization $k= 10$) \acrshort{epsdp} run time increases linearly with the number of vertices. Indeed, throughout the benchmark instances tested, the run time does not exceed a few seconds per each of the test cases. At the same time, the bound is often  almost as good as that of the Branch and Bound Biq-Mac Solver of \cite{krislock2014improved},  which requires a significant amount of time. 

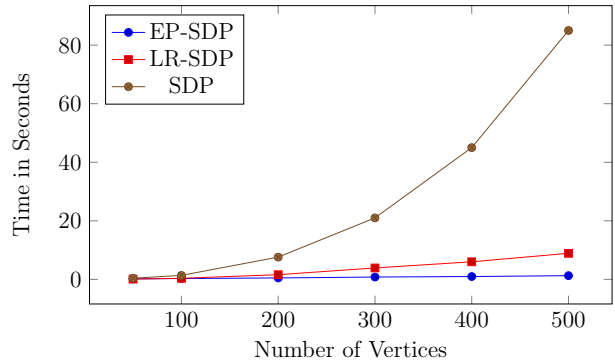
\begin{figure}[t!]
\centering
      \begin{tikzpicture}[scale=0.73]
    \begin{axis}[legend pos=north west, title={},xlabel=Number of Vertices, ylabel=Time in Seconds]
         \addplot table [x=vert, y=EPSDP, col sep=comma] {tables/epsdp/time.txt};
	  \addlegendentry{EP-SDP}
         \addplot table [x=vert, y=SDPLR, col sep=comma] {tables/epsdp/time.txt};
         \addlegendentry{LR-SDP}
         \addplot table [x=vert, y=SDP, col sep=comma] {tables/epsdp/time.txt};
         \addlegendentry{SDP}
    \end{axis}
    \end{tikzpicture} 
\caption[Time complexity for Erdos-Renyi graphs (chart)]{Run time comparison for Erdos-Renyi random graphs of different size. Here we compare only \acrshort{sdp}-based approaches and show that the \acrshort{epsdp} is the fastest and it's time complexity is linear in the number of vertices.}\label{time}
\end{figure}

\begin{table}[!t]
\label{tbl:time}
\centering
\begin{tabular}{l|c|c|c}
Instance &  {\bf EP-SDP} & {\bf LR-SDP} & {\bf SDP}\\
\hline\hline
E-R(50, 0.2) & 0.2s &  0.1s   &   0.4s  \\
\hline
E-R(100, 0.2) & 0.3s &  0.4s   &  1.4s  \\
\hline
E-R(200, 0.2) &  0.5s &  1.6s   &   7.6s  \\
\hline
E-R(300, 0.2) &  0.8s &  3.9s   &   21.0s  \\
\hline
E-R(400, 0.2) &  1.0s &  6.0s   &   45.0s  \\
\hline
E-R(500, 0.2) &  1.3s &  8.9s   &   85.0s  \\
\end{tabular}
\caption[Time complexity for Erdos-Renyi graphs (table)]{Run time comparison for Erdos-Renyi random graphs of different size (the number of vertices). Here we compare only \acrshort{sdp}-based approaches and show that the \acrshort{epsdp} is the fastest. The results are averaged over 500 samples for every row.}
\end{table}

\section{Conclusions}\label{sec:conclusion}

This chapter presented a unified view of the penalty functions used in low-rank semidefinite programming using entropy as a penalty. 
This makes it possible to find a low-rank optimum, where there are optima of multiple ranks.
Semidefinite programs with an entropy penalty can be solved efficiently using first-order optimization methods with linear-time per-iteration complexity, which makes them applicable to large-scale problems that appear in machine learning and polynomial optimization. 
Our case study illustrated the practical efficiency on binary MAP inference problems. 
The next step in this direction is to consider the structure of the \acrshort{sdp}, which seems to be  crucial for further scalability. 
\chapter{Ising Model Control}
\label{cap:mitigation}

\section{Introduction}

We follow the previous work \cite{GMPandemic} in justification for the use of the \acrfull{gm} to study and mitigate pandemics. Therefore, we start from providing a brief recap of the prior literature on modeling of the epidemics, describe the logic which led us in \cite{GMPandemic} to the Ising Model (IM) formulation, and then state formally the inference and prevention problems addressed in the manuscript.

Difficulty in both predicting and neutralizing the spread of pandemics is a major social challenge of humanity. Technically speaking, we are yet to design a coherent data lifecycle for modeling and prevention both in terms of the global strategies and local tactics. To address the challenge, we must devise a hierarchy of spatio-temporal models with different resolutions -- from individual to community, county to the city, and from the moment a pathogen first enters our bodies, to days of disease development and to community transmission. Importantly, the models should be efficient in computing probabilistic predictions (for instance, offering the marginal probability heat map for the city neighborhoods to transition from the current/prior state of infection to the projected/a-posteriori state in two weeks).  

Epidemiology and Mathematical Biology experts have relied in the past on a number of modeling approaches. The \acrfull{abm}, introduced in epidemiology in 2004-2008~\cite{eubank_modelling_2004,2005Longini,ferguson_strategies_2005,ferguson_strategies_2006,2006Germann,2008Halloran}, have complemented the earlier  compartmental models~\cite{1910Ross,1927Kermack,1991Anderson,2000Hethcore}. Using \acrshort{abm}s, even though not exclusive to epidemiology \cite{ABM-wikipedia,2018Downey}, became a breakthrough in the field, as they allowed to make a significant improvement in the quality of predictions, especially in the spatio-temporal resolution of how the disease spreads and how one can mitigate its spread. The models became and remained a core part of the epidemiology data life-cycle. (See for instance \cite{ABM-Columbia,ABM-Gates} for most recent bibliography.) The \acrshort{abm}s provide a detailed prediction of how pandemics spread within counties, cities, and regions.  A majority of the country-, city- or county- scale testbeds testing various mitigation strategies are resolved nowadays with \acrshort{abm}s. In particular, recently \acrshort{abm}s have been used extensively to inform public health in (non-pharmaceutical) interventions against the spread of COVID-19 ~\cite{2020Ferguson,eubank_commentary_2020,lanl2020covid,2020ABM-calibration,kaxiras2020multiple}, and verify new strategies like test-trace-quarantine \cite{ABM-Gates}, among many other applications. 

There are two major problems with the modeling of pandemic.  First, many parameters need to be calibrated on data.  Second, even when calibrated for the current state of pandemic the models which are too detailed become impractical for making a forecast and for developing prevention strategies -- both requiring checking multiple (forecast and/or prevention) scenarios. Using \acrshort{abm}s, which are clearly over-modeled (too detailed) is especially problematic in the context of the latter.

For example, the open-source \acrshort{abm} solver FLUTE~\cite{chao2010flute} developed originally for modeling influenza, works with data that are acquired through Geographic Information Systems (GIS) on the scale of census tracts or communities, which is a very reasonable scale of spatial resolution to understand the dynamics of pandemics on a local scale. FLUTE populates each of the communities with thousands to millions of inhabitants in order to account for their daily patterns of travel. We believe that constructing effective \acrfull{gm} of Pandemics with community-scale spatial resolution and then modeling pairwise (and possibly higher-order) epidemic interactions between communities directly, without introducing the thousands-to-millions of dummy agents, will complement (as discussed in the next paragraph), but also improve upon \acrshort{abm}s by being more efficient, robust and easier to calibrate.

An important, and possibly one of the first, \acrfull{gm} of the COVID-19 pandemic was proposed in \cite{chang_mobility_2020}. Dynamic bi-partite \acrshort{gm}s connecting census tracts to specific Points Of Interest (non-residential locations that people visit such as restaurants, grocery stores and religious establishments) within the city and studying dynamics of the four-state (Susceptible, Exposed, Infectious and Removed) of a census tract (graph node) on the graph, were constructed in \cite{chang_mobility_2020}  for major metro-area in USA based on the SafeGraph mobility data \cite{SafeGraph-Mobility}.

In fact, similar dynamic \acrshort{gm}s, e.g. of the \acrfull{icm} type  \cite{2003KempeKleinbergTardos,2012NetrapalliSanghavi,2012GomezLeskovecKrause,KhaDilSon13,2016RosenfeldNitzanGloberson}, were introduced even earlier in the CS/AI literature in the context of modeling how the rumors spread over social networks (with a side reference on using \acrshort{icm} in epidemiology). As argued in \cite{GMPandemic} the \acrfull{icm} can be adapted to modeling pandemics. (Another interesting use of the \acrshort{icm} to model COVID-19 pandemic was discussed in \cite{Chen_2020}.) In its minimal version, an \acrshort{icm} of Pandemic can be built as follows. Assume that the virus spreads in the community (census tract) sufficiently fast, say within five days -- which is the estimate for the early versions of COVID-19 median incubation period. If an infected person enters a community/neighborhood but does not stay there,  he infects others with some probability. If a single resident of the community becomes infected, all other residents are assumed infected as well (instantaneously). The model is a discrete-time dynamic model in which nodes in a network are in one of the three states: {\bf S}usceptible, {\bf I}nfected, or {\bf R}emoved. The nodes represent communities/neighborhoods. 
A contact between an {\bf I}nfected community/node and another community which is {\bf S}usceptible has an assigned probability of disease transmission, which can also be interpreted as the probability of turning the {\bf S} state into {\bf I} state.  Consistently with what was described above, the network is represented as a graph, where nodes are tracts and edges, connecting two tracts, have an associated strength of interaction representing the probability for the infection to spread from one node to its neighbor. A seed of the infection is injected initially at random, for example, mimicking an exogenous super-spreader infection event in the area; examples could include political or religious gatherings. See Figure~\ref{fig:ICM} illustrating dynamics of the cascade model over 3-by-3 grid  graph. Color coding of nodes is according to {\bf S}usceptible=blue, {\bf I}nfected=red, {\bf R}emoved=black. Given the starting infection configuration, each infected community can infect its graph-neighbor community during the next time step with the probability associated with the edge connecting the two communities. Then the infected community moves into the removed state. The attempt to infect each neighbor is independent of all other neighbors. This creates a cascading spread of the virus across the network. The cascade stops in a finite number of steps, thereby generating a random {\bf R}emoved pattern, shown in black in the Fig.~\ref{fig:ICM},  while other communities which were never infected (remain {\bf S}usceptible) are shown in blue.

\begin{figure*}[t]
\centering
\includegraphics[width=0.2\textwidth]{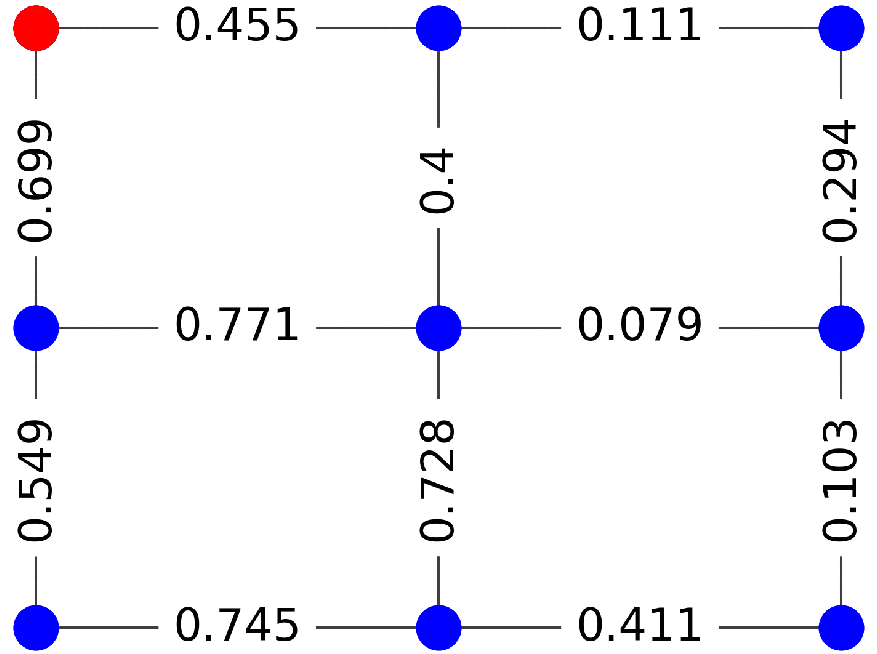}\hfill
\includegraphics[width=0.2\textwidth]{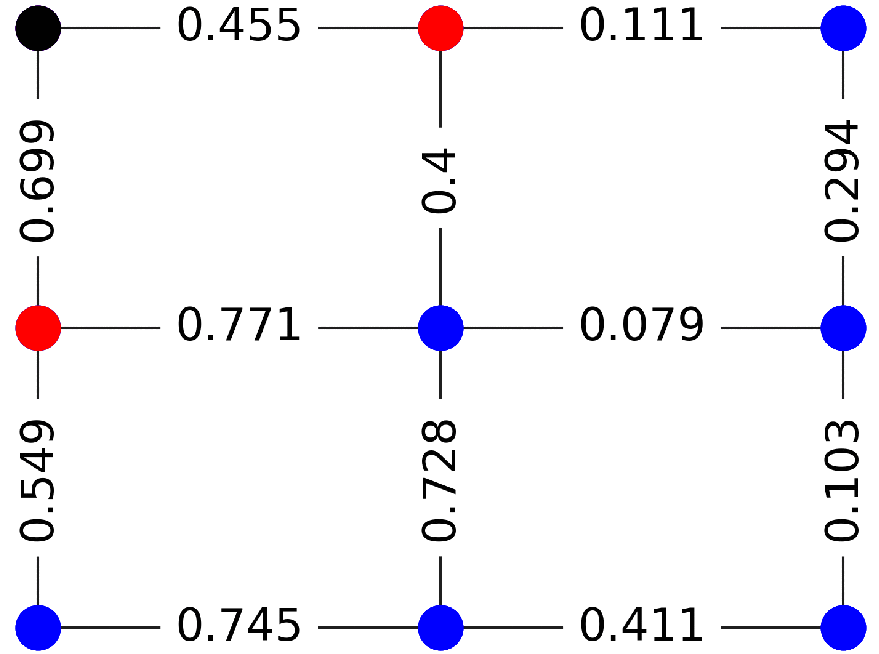}\hfill
\includegraphics[width=0.2\textwidth]{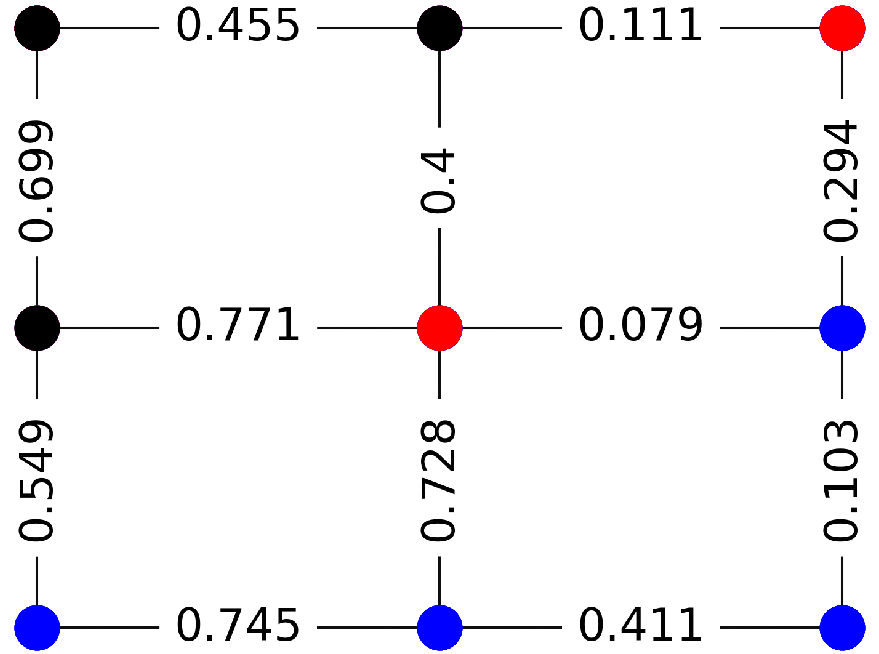}\hfill
\includegraphics[width=0.2\textwidth]{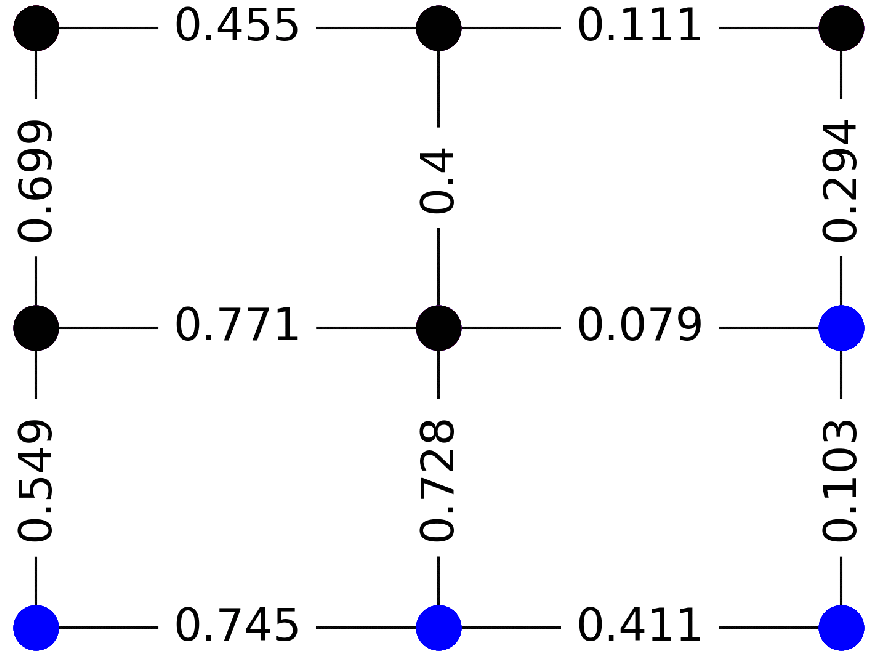}
\caption[Ising model of pandemic]{An exemplary random sequence (top-left to top-right to bottom-left to bottom-right) of the Independent Cascade Model (ICM) dynamics over $3\times3$ grid. Nodes colored red, blue, and black are {\bf I}nfected, {\bf S}usceptible, and {\bf R}emoved at the respective stage of the dynamical process. This (shown) sample of the dynamic process terminates in 3 steps. \acrfull{imp}, which is the focal point of this manuscript, describes a regularized version of the ICM terminal state, where only two states ({\bf S}-blue and {\bf R}-black) are left. (See text for details.)}
\label{fig:ICM}
\end{figure*}

It was shown in \cite{GMPandemic} that with some regularization applied, statistics of the terminal state of the Cascade Model of Pandemic turns into a Graphical Model of the attractive Ising Model type. 

\underline{\bf This chapter Road Map.}  Working with the Ising Model of Pandemic, we start the technical part of the manuscript by posing the {\bf Inference/Prediction Challenge} in Section \ref{ref:IMP}. Here,  the problem is stated,  first, as the Maximum A-Posteriori over an attractive Ising Model,  and we argue,  following the approach which is classic in the \acrshort{gm} literature,  that problem can be re-stated as a tractable \acrshort{lp}. We then proceed to Section  \ref{sec:Prevention} to pose the main challenge addressed in the manuscript -- the {\bf Prevention Challenge} -- as the two-level optimization with inner step requiring resolution of the aforementioned Prediction Challenge. Aiming to reduce the complexity of the Prevention problem, we turn in Section \ref{sec:Geometry} to the analysis of the conditions in the formulation of the Prediction Challenge, describing the Safety domain in the space of the Ising Model parameters. We show the Safety domain is actually a polytope, even though exponential in the size of the system.  We proceed in Section \ref{sec:Projection} with analysis of the Prevention Challenge,  discussing the interpretation of the problem as a projection to the Safety Polytope from the polytope exterior, needed when the bare prediction suggests that system will be found with high probability outside of the Safety Polytope. Section \ref{sec:Polarized} is devoted to approximation which allows an enormous reduction in the problem complexity. We suggest here that if the graph of the system is sufficiently dense, the resulting \acrshort{map} solution may only be in one of the two polarized states (a) completely safe (no other nodes except the initially infected) pick the infection, or (b) the infection is spread over the entire system. We support this remarkable simplification by detailed empirical analysis and also by some theoretical arguments. Section \ref{sec:Experiments} is devoted to the experimental illustration of the methodology on the practical example of the Graphical Model of Seattle. The manuscript is concluded in Section \ref{sec:Conclusions} with a brief summary and discussion of the path forward.

\section{Ising Model of Pandemic}\label{ref:IMP}

As argued in \cite{GMPandemic} the terminal state of a dynamic model generalizing the \acrshort{icm} model can be represented by the \acrfull{imp}, defined over graph $\cal G = ({\cal V}, {\cal E})$, where ${\cal V}$ is the set of $N=|{\cal V}|$ nodes and ${\cal E}$ is the set of undirected edges. The \acrshort{imp}, parameterized by the vector of the node-local biases, $h=(h_a|a\in\cal V)\in \mathbb{R}^N$, and by the vector of the pair-wise (edge) interactions, $J=(J_{ab}|\{a,b\}\in {\cal E})$, 
describes the following Gibbs-like probability distribution for a state, $x=(x_a=\pm 1|a\in{\cal V})\in 2^{|\cal V|}$, associated with $\cal V$:
\begin{equation}
P(x~ \vert~ J, h) = \frac{\exp{(-E(x~ \vert~ J, h))}}{Z(J, h)}, \label{eq:Gibbs}
\end{equation}
where any node, $a\in \cal V$ can be found in either {\bf S}- (susceptable, never infected) state,  marked as $x_a=-1$, or {\bf R}- (removed, i.e. infected prior to the termination) state, marked as $x_a=+1$.  In Eq.~(\ref{eq:Gibbs}), $E(x| J, h)$ and $Z(J, h)$ are model's energy function and partition function respectively:
\begin{eqnarray}\label{eq:E}
E(x~ \vert~ J, h) &=& \sum\limits_{a\in \cal V} h_a x_a - \sum\limits_{a, b\in \cal V} J_{ab} x_a x_b,\\ \label{eq:pf}
Z(J, h) &=& \sum\limits_{x} \left (\sum\limits_{a\in \cal V} h_a x_a - \sum\limits_{a, b\in \cal V} J_{ab} x_a x_b \right ). 
\end{eqnarray}

In what follows, we will focus on finding the \acrfull{map} state of the \acrshort{imp} conditioned to a particular initialization -- setting a subset of nodes, ${\cal I}\in {\cal V}$,  to be infected. We coin the \acrshort{map} problem {\bf Inference Challenge}:
\begin{gather}\label{eq:InfQ}
x^{(\text{MAP})}({\cal I}~ \vert~ J,h)=\text{arg}\min\limits_{x} E(x~ \vert~ J, h)\Big|_{\forall a\in {\cal I}:\quad  x_a = +1},
\end{gather}
where we emphasize dependence of the \acrshort{map} solution on the set of the initially infected nodes, ${\cal I}$. 

Note that in general finding $x^{(\text{MAP})}$ is NP-hard \cite{Barahona_1982}.  However  if $J>0$ element-wise, i.e. the Ising Model is attractive (also called ferromagnetic in statistical physics), Eq.~(\ref{eq:InfQ}) becomes equivalent to a tractable (polynomial in $N$) Linear Programming (see \cite{14ZWP} and references therein).  Notice a few other tractable cases for graphs with specific topology~\cite{chertkov2020gauges,likhosherstov2020tractable,likhosherstov2019inference,likhosherstov2019new}
In fact, the \acrshort{imp} is attractive, reflecting the fact that the state of a node is likely to be aligned with the state of its neighbor. 

Let us also emphasize some other features of the \acrshort{imp}:
\begin{enumerate}
    \item $\cal G$ should be thought of as an "interaction" graph of a city,  reflecting transportation, commutes, and other forms of interactions between populations with the homes at the two nodes (census tracts) linked by an edge. The strength of a particular $J_{ab}$ shows the level of interaction associated with the edge $\{a,b\}$. 
    \item A component, $h_a$, of the vector of local biases, $h$, is reflecting $a$-node specific factors such as immunization level, imposed quarantine, and degree of compliance with the public health measures (e.g., wearing masks and following other rules). Large negative/positive $h_a$ shows that residents of the census tract associated with the node $a$ are largely healthy/infected. 
\end{enumerate}
If solution of the Inference Challenge problem is such that the ${\bf R}$-subset of the \acrshort{map} solution, $x^{(\text{MAP})}({\cal I}|J,h)$, i.e. 
\begin{gather}\label{eq:V-I}
    {\cal R}({\cal I},J,h)=\Big\{a\in {\cal V} \, | \, x_a^{(\text{MAP})}({\cal I}|J,h)=+1\Big\},
\end{gather}
is sufficiently large, we would like to mitigate the infection, therefore setting the Prevention Challenge discussed in the next Section.

\section{Prevention Challenge}\label{sec:Prevention}

Let us assume that modification of $J$ and $h$ are possible and consider the space of all feasible $J$ and $h$. We will then identify {\it Safe Domain} as a sub-space of feasible $J$ and $h$ such that for all the initial sets of the initially infected nodes, ${\cal I}$, considered the resulting "infected" subset, ${\cal R}({\cal I},J,h)$, is sufficiently small. A more accurate definition of the Safe Domain follows.  Then, we rely on the definition to formulate the control/mitigation problem coined Prevention Challenge. At this stage, we would also like to emphasize that studying the geometry of the Safe Domain is one of the key contributions of this manuscript.

{\bf Definition.} 
Consider \acrshort{imp} over ${\cal G}=({\cal V},{\cal E})$ and  with the parameters $(J,h)$.  Let us also assume that the set of initially infected nodes, ${\cal I}$, is drawn from the list, $\Upsilon$. We say that $(J, h)$ is in the $k$-{\it Safe Domain} if for every ${\cal I}$ from $\Upsilon$ the number of ${\bf R}$-nodes in the \acrshort{map} solution (\ref{eq:InfQ}), is at most $k$, i.e. \begin{gather}\label{eq:k-safe}
    \forall {\cal I}\in \Upsilon:\quad |{\cal R}({\cal I},J,h)|\leq k,
\end{gather}
where ${\cal R}({\cal I},J,h)$ is defined in Eq.~(\ref{eq:V-I}). 

{\bf Prevention Challenge:} Given $(J^{(0)},h^{(0)})$ describing the bare status of the system (city) which is not in the $k$-{\it Safe Domain}, and given the cost of the $(J,h)$ change, $C\left((J,h);(J^{(0)},h^{(0)})\right)$, what is least expensive change to $(J^{(0)},h^{(0)})$ state of the system which is in the the $k$-{\it Safe Domain}? Formally,  we are interested to solve the following optimization:
\begin{gather}\label{eq:prevention-Q}
    (J^{(\text{\tiny corr})},h^{(\text{\tiny corr})})=\text{arg}\min\limits_{(J,h)} C\left((J,h);(J^{(0)},h^{(0)})\right)_{\text{Eq.~(\ref{eq:k-safe})}}.
\end{gather}

Expressing it informally, the Prevention Challenge seeks to identify a minimal correction (thus "corr" as the upper index) $(J^{(\text{\tiny corr})},h^{(\text{\tiny corr})})$, which will move the system to the safe regime from the unsafe bare one, $(J^{(0)},h^{(0)})$.  The measures may include limiting interaction along some edges of the graph, thus modifying some components of $J$, or enforcing local biases, e.g., increasing level of vaccination, at some component of $h$.

Given that condition in Eq.~(\ref{eq:k-safe}) also requires solving Eq.~(\ref{eq:InfQ}) for each candidate $(J,h)$, the {\bf Prevention Challenge} formulation is a difficult two-level optimization.
However, as we will see in the next Section, the condition in Eq.~(\ref{eq:k-safe}) (and thus the inner part of the aforementioned two-level optimization) can be re-stated as the requirement of being inside of a polytope in the $(J,h)$ space. In other words, the $(k)$-{\it Safe Domain} is actually a polytope in the $(J,h)$ space.

\section{Geometry of the MAP States}\label{sec:Geometry}

Before solving the Prevention Challenge problem, we want to shed some light on the geometry of the \acrshort{map} states. We work here in the space of all the Ising models over a graph ${\cal G}=({\cal V},{\cal E})$, where each of the models is specified by $(J,h)$. 

{\bf Proposition.} Safe Domain of a graph ${\cal G}=({\cal V},{\cal E})$ with $N=|{\cal V}|$ nodes is a polytope in the space of all feasible parameters, $(J,h)$, defined by an exponential in $N$ number of linear constraints.

{\bf Remark.} The Proposition allows us, from now on, to use {\it Safe Polytope} instead of the {\it Safe Domain}.

{\bf Proof of the Proposition.} The space of all the Ising models is divided into $2^N$ regions by the corresponding \acrshort{map} states. Moreover, the boundary between any pair of neighboring regions is linear: consider two states $x^{(i)}$ and $ x^{(j)}$, and denote $(J,h)^{(i)}$ (resp.  $(J,h)^{(j)}$) the set of all the Ising models with the \acrshort{map} state $x^{(i)}$ (resp. $x^{(j)}$), then $(J,h)^{(i)}$ and $(J,h)^{(j)}$ are separated by the equation, $E(x^{(i)}~ \vert~ J,h) = E(x^{(j)}~ \vert~ J,h)$, which is linear in $(J,h)$. For a subset, $R\subseteq \cal V$, of nodes, let $x^{(R)}$ be the state in which, $x_a=+1,\ \forall a\in R, x_a=-1,\ \forall a\notin R$. Let $X^{(R)}$ be the set
of all the \acrshort{map} states, $x$, such that $\forall a\in R,\ x_a=+1$ (while other nodes, i.e. $b\in {\cal V}\setminus R$, are not constrained, $x_b=\pm 1$). Then the $k$-Safe Polytope, which we denote, $\text{SP}(k)$, is defined by at most $\sum\limits_{k^\prime=1}^{k}\binom{N}{k^\prime}\cdot(2^{N-k^\prime}-1) $ linear inequalities:
\begin{equation}\label{eq:SP-k}
\text{SP}(k) = \hspace{-0.9cm}\bigcap\limits_{\begin{array}{c}\forall R,|R|\le k;\\ \forall x\in X^{(R)}\setminus x^{(R)}\end{array}}\hspace{-1.25cm} \left\{(J, h)~ \vert~ E(x^{(R)} | J, h) > E(x~ \vert~ J, h)\right\},
\end{equation}
were some of these linear inequalities on the rand hand side may be redundant.

{\bf Remark.} In the case of $k=1$ (which, obviously, applies only if all the initial infections are at a single nodes, i.e. $\forall {\cal I}\in \Upsilon,\ |{\cal I}|=1$), there are at most, $N \cdot (2^{N-1} - 1)$ linear inequalities.

We illustrate the geometry of the Ising model over the triangle graph (three nodes connected in a loop, $K_3$) in Fig.~\ref{fig:geometry} and Fig.~\ref{fig:SafePoly}. For both illustrations, we fix the $h$ value to $-1$ at all the nodes, and we are thus exploring the remaining three degrees of freedom, $J_{12},J_{13},J_{23}$ (since $J$ is symmetric),
which corresponds to exploring interactions within the class of attractive Ising models, $\forall a,b=1,2,3:\ J_{ab}\in \mathbb{R}_+$. 

First, we consider the case when the only node $a=1$ is infected. In this simple setting there are four possible \acrshort{map} states, 
\[
(x_1,x_2,x_3)\in \{(+1, -1, -1), (+1, -1, +1), (+1, +1, -1),(+1,+1,+1)\}
\] 
shown in Fig.~(\ref{fig:geometry}) as green, blue, yellow and red, respectively. Finally, in the figure Fig.~(\ref{fig:SafePoly}) we plot the Safe Polytope $\text{SP}(1)$.
We observe that the two "polarized" \acrshort{map} states, $(+1,-1,-1)$ and $(+1,+1,+1)$, are seen most often among the samples, while domain occupied by the other two "mixed" \acrshort{map} states, $(+1,-1,+1)$ and $(+1,+1,-1)$ is much smaller, with the two modes  positioned on the interface between the two polarized states.

As will be shown below in the next Section, the polarization phenomena with only two "polarized" \acrshort{map} states, which we coin in the following the two polarized modes, which we see on this simple triangle example, is generic for the attractive Ising model.  

\begin{figure}[t]
\centering
\includegraphics[width=0.9\columnwidth]{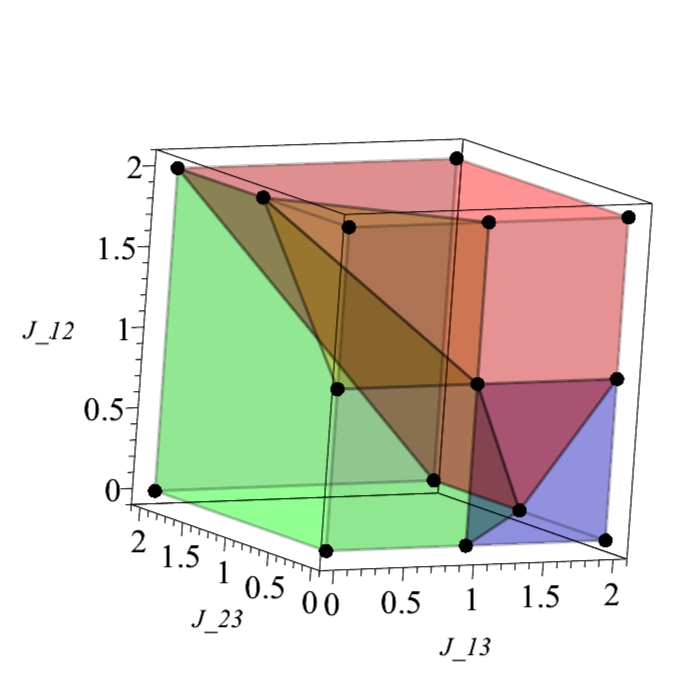} 
\caption[Geometry of the attractive Ising model]{Geometry of the attractive Ising model illustrated  on the example of a triangle graph ($K_3$) when a single node is infected. See explanations in the text.}
\label{fig:geometry}
\end{figure}

\begin{figure}[t]
\centering
\includegraphics[width=0.9\columnwidth]{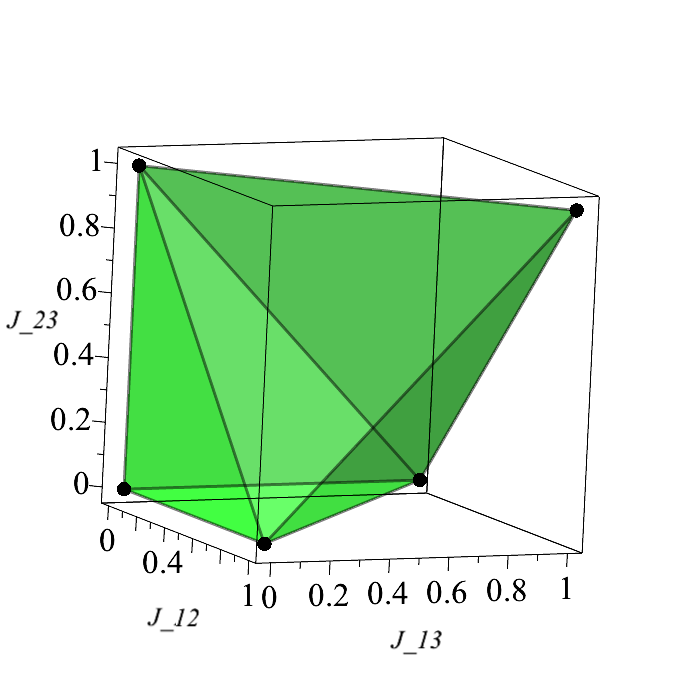} 
\caption[The safe polytope]{The Safe Polytope illustrated on the example of a triangle graph ($K_3$) with field vector $h = [-1, -1, -1]$. See explanations in the text.}
\label{fig:SafePoly}
\end{figure}

\section{Two Polarized Modes}\label{sec:Polarized}

{\bf Definition.} Consider a particular subset of the initially infected nodes, ${\cal I}$ (where thus, $\forall a\in{\cal I}:\ x_a=+1$). We call the \acrshort{map} state of the model  {\it polarized} when one of the following is true:  {\bf (i) }  only initially infected nodes show $+1$ within the \acrshort{map} solution, $\forall a\in {\cal I}:\ x_a=+1,\ \forall b\in{\cal V}\setminus{\cal I}:\ x_b=-1$ or {\bf (ii)} all nodes within the \acrshort{map} state show $+1$, $\forall a\in{\cal V}:\ x_a=+1$. We call a \acrshort{map} state {\it mixed} otherwise.

Experimenting with many dense graphs, which are typical in the pandemic modeling of modern cities with extended infrastructures and multiple destinations visited by many inhabitants, we observe that the two polarized \acrshort{map} states dominate generically,  while the mixed states are extremely rare. 

Fig.~\ref{fig:density} illustrates results of one our ensemble of random \acrshort{imp}s' experiments. We, first, fix $N$ to $20$, pick $M$ such that $M\le N(N-1)/2=190$ and then generate at random $M$ edges connecting the $20$ nodes. 
Then, for each of the random graphs (characterized by its own $M$) we generate $500$ random samples of $(J,h)$, representing attractive Ising models. Finally, we find the \acrshort{map} state for each \acrshort{imp} instance, count the number of mixed states and show the dependence of the fractions of the mixed states (in the sample set) in the Fig.~(\ref{fig:density}). A fast decrease of the proportion of the mixed states is observed with an increase in $M$.  

\begin{figure}[t]
\centering
\includegraphics[width=0.9\columnwidth]{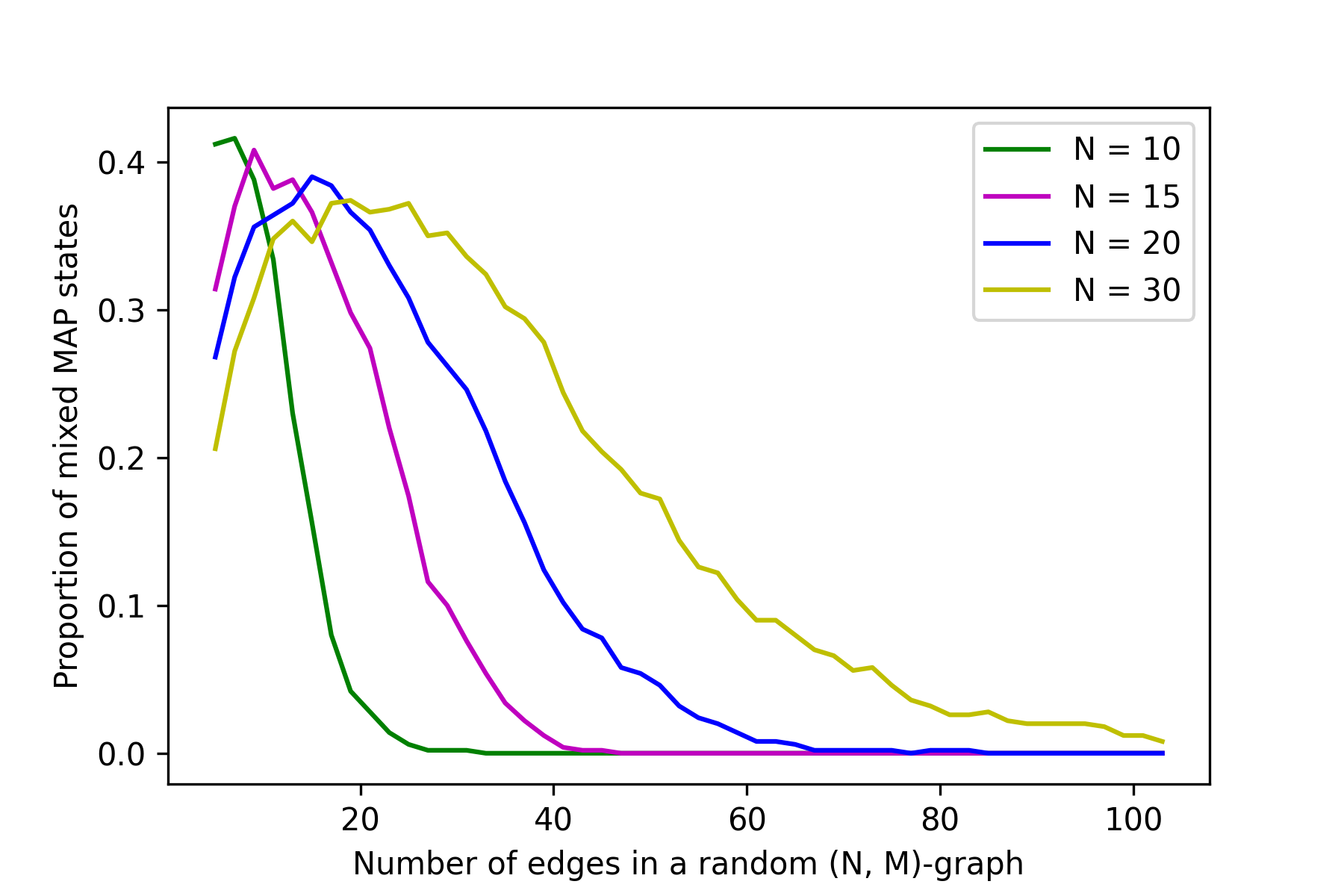} 
\caption[Proportion of the mixed states in all samples]{Proportion of the mixed states in all samples for an ensemble of the (attractive) Ising Model of Pandemic over graphs with $N$ nodes, shown as a function of the varying number of edges, $M$. Each shown point is the result of the averaging over $500$ random instances of the $(J,h)$ over the same graph. (See text for additional details.)} 
\label{fig:density}
\end{figure}

\begin{figure}[t]
\centering
\includegraphics[width=0.9\columnwidth]{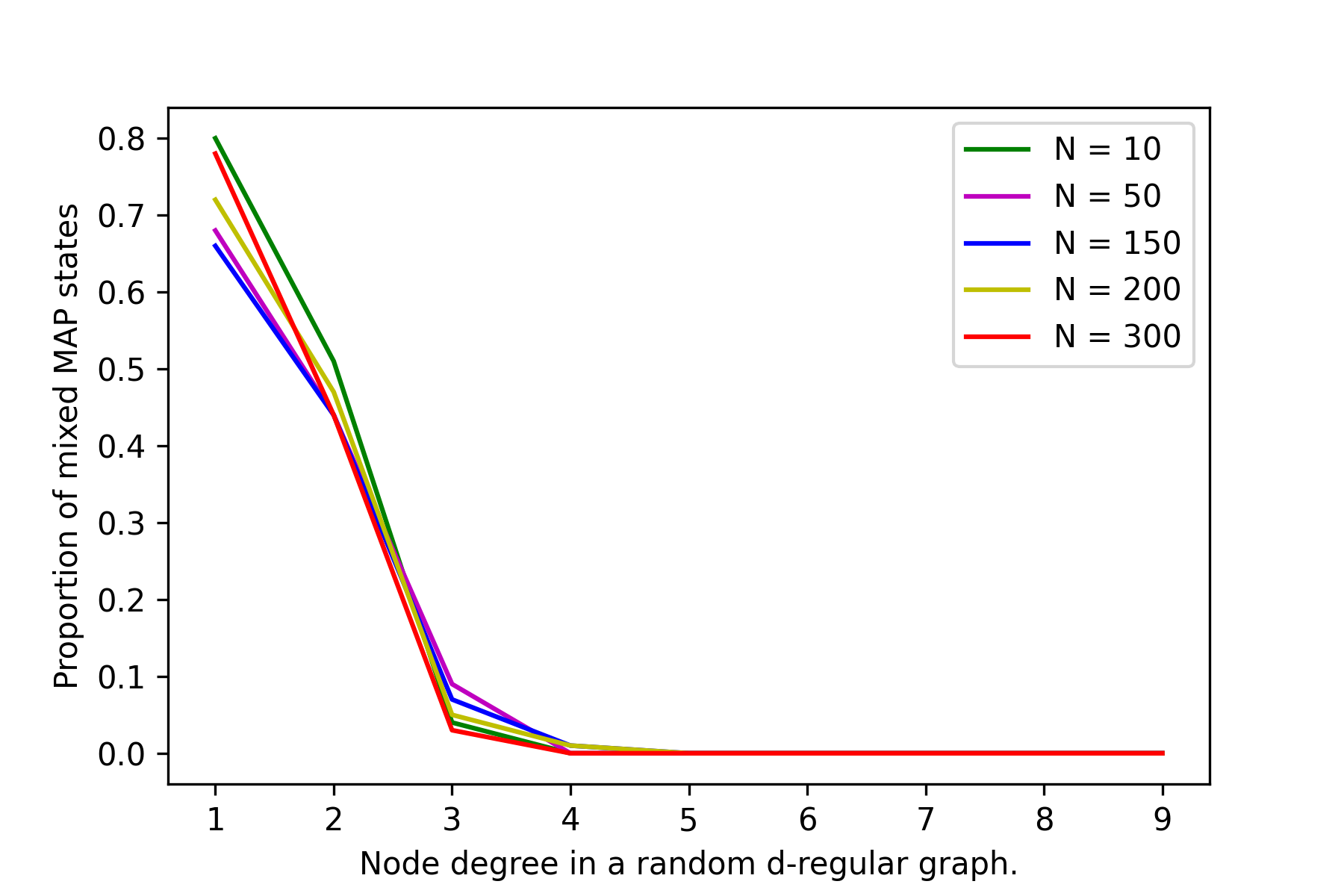} 
\caption[Proportion of the mixed-to-polarized states]{Proportion of the mixed-to-polarized states for an ensemble of the (attractive) Ising Model of Pandemic over $d$-regular graphs with $N$ nodes, shown as a function of $d$. Each point is the result of  averaging over $100$ random instances of the $(J,h)$ over different random graphs with the same node degree. (See text for additional details.)}
\label{fig:density_p}
\end{figure}

Extension of these experiments (see Fig.~(\ref{fig:density_p})) suggests that when we consider an ensemble of \acrshort{imp}s over graphs with $N$ nodes and the average degree $\alpha=O(1)$ which is sufficiently large (so that the graph is sufficiently dense) and increase $N$, we observe that the \acrfull{msp}, or equivalently proportion of the mixed-to-polarized states, decreases dramatically. Moreover, based on the experiments, we conjecture that the MSP decays to zero at $\alpha>\alpha_c$, but it saturates at $\alpha<\alpha_c$,  where $\alpha_c$ is the threshold depending on the ensemble details.  This threshold behavior is akin to the phase transition that occurred in many models of the spin glass theory \cite{MezardParisiVirasoro} and many models of the Computer Science and Theoretical Engineering defined over random graphs and considered in the thermodynamic limit, i.e. at $N\to \infty$. See e.g. \cite{RichardsonUrbanke} (application in the Information Theory, and specifically in the theory of the Low Density Parity Check Codes) \cite{MezardMontanari} (applications in the Computer Science, and specifically for random SAT and related models) and references therein. We postpone further discussions of the conjecture for a future publication (see also brief discussion in Section \ref{sec:Conclusions}). 

We will continue discussion of the two-mode solution in the next Section. 

\section{Projecting to the Safe Polytope}\label{sec:Projection}

In this Section we aim to summarize all the findings so far to resolve the Prevention Challenge formulated in Section \ref{sec:Prevention}, specifically in  Eq.~(\ref{eq:prevention-Q}) stating the problem as finding a minimal projection to the Safety Domain/Polytope from its exterior. The task is well defined, but in general, and as shown in Section \ref{sec:Geometry}, it is too complex  -- as the description of the Safety Polytope (number of linear constraints, required to define it) is exponential in the system size (number of nodes in the graph).  However, the two-mode approximation, introduced in Section \ref{sec:Polarized}, suggests a path forward: use the two-mode approximation and therefore remove all the linear constraints but one, separating the two polarized states.

Let us denote the two-mode approximation of the Safe Polytope by $\widehat{SP}(\Upsilon)$, where thus $k$ in the original Safe Polytope, $SP(k)$, is replaced by the set $\Upsilon$ of all the initial infection patters. Then we write,
\begin{gather}\label{eq:SP-Upsilon}
    \widehat{SP}(\Upsilon)=\bigcap\limits_{{\cal I}\in \Upsilon}\big\{(J, h)~ \vert~ E(+1^{2N} ~ \vert~ J, h) \geq  E(x^{({\cal I})}~ \vert~ J, h)\big\},
\end{gather}
where, $\forall a\in{\cal I}:\ x^{({\cal I})}_a=+1$ and $\forall b\in {\cal V}\setminus {\cal I}:\ x^{({\cal I})}_b=-1$. Eq.~(\ref{eq:SP-Upsilon}) represents a polytope stated in terms of the $|\Upsilon|$ constraints. In particular, if $\Upsilon$ accounts for all the initial infections, ${\cal I}$, of size not large than $k$, then $|\Upsilon|=\sum_{k^\prime=1}^k \binom{N}{k^\prime}$: the number of constraints grows exponentially in the maximal size of the initial infections, however the number of the constraints remains tractable for any $k=O(1)$. Replacing conditions in Eq.~(\ref{eq:prevention-Q}) by  $\widehat{SP}(\Upsilon)$, defined in  Eq.~(\ref{eq:SP-Upsilon}), one arrives at the following tractable (in the case of $k=O(1)$) convex optimization expression answering the Prevention Challenge approximately (within the two-mode approximation):
\begin{gather}\label{eq:prevention-Q-two-mode}
    (\widehat{J}^{(\text{\tiny corr})},\widehat{h}^{(\text{\tiny corr})})=\text{arg}\min\limits_{(J,h)} C\left((J,h);(J^{(0)},h^{(0)})\right)_{\text{Eq.~(\ref{eq:SP-Upsilon})}}.
\end{gather}
This formula is the final result of this manuscript analytic evaluation. 
In the next Section we use Eq.~(\ref{eq:prevention-Q-two-mode}), with $C(\cdot;\cdot)$ substituted by the $l_1$-norm, to present the result of our experiments in a quasi-realistic setting describing a (hypothetical) pandemic attack and optimal defense, i.e., prevention scheme.

\section{Experiments}\label{sec:Experiments}


\subsection{Seattle data}

We illustrate our methodology on a case study of the city of Seattle. Seattle has 131 Census Tracts. (Each Census Tract includes 1 to 10 Census Block Groups with  600 to 3000 residents.) Each Census Tract represents 1200 to 8000 population, and its boundaries are designed to represent natural or urban landmarks and also to be persistent over a long period \cite{USBureauGlossary}. To reduce complexity, we merge census tracts into 20 regions. See Fig.~~\ref{fig:seattle20area}. To prepare this splitting of Seattle into 20 regions/nodes, we utilize geo-spatial information from the TIGER/Line Shapefiles project provided by U.S. Census Bureau \cite{USBureauTiger}. The travel data of Seattle was extracted from the Safegraph dataset \cite{SafeGraph-DataConsortium}, which provides anonymized mobile tracking data. Each data point in the Safegraph database describes the number of visits from a Census Block to a specific point of interest represented by latitude and longitude.   
Mobility data associated with travelers crossing the boundaries of Seattle was ignored.  
We then follow the methodology developed in \cite{GMPandemic} to combine the aggregated travel data with the epidemiological data, representing current state of infection in the area. This results in the estimation of the pair-wise interactions, $J$, parameterizing the Ising Model of Pandemic. We also come up with an exemplary (uniform over the system) local biases, $h$, completing the definition of the model.  (We remind that the prime focus of the manuscript is on developing methodology which is AI sound and sufficiently general. Therefore,  the data used in the manuscript are roughly representative of the situation of interest, however not fully practical.) We consider a situation with different levels of infection and chose $(J^{(0)},h^{(0)})$ stressed enough,  that is resulting in the prediction (answer to our Prediction Challenge), which lands the system in the dangerous domain -- outside of the Safety Polytope. 
\begin{figure}[t]
\centering
\includegraphics[width=0.9\columnwidth]{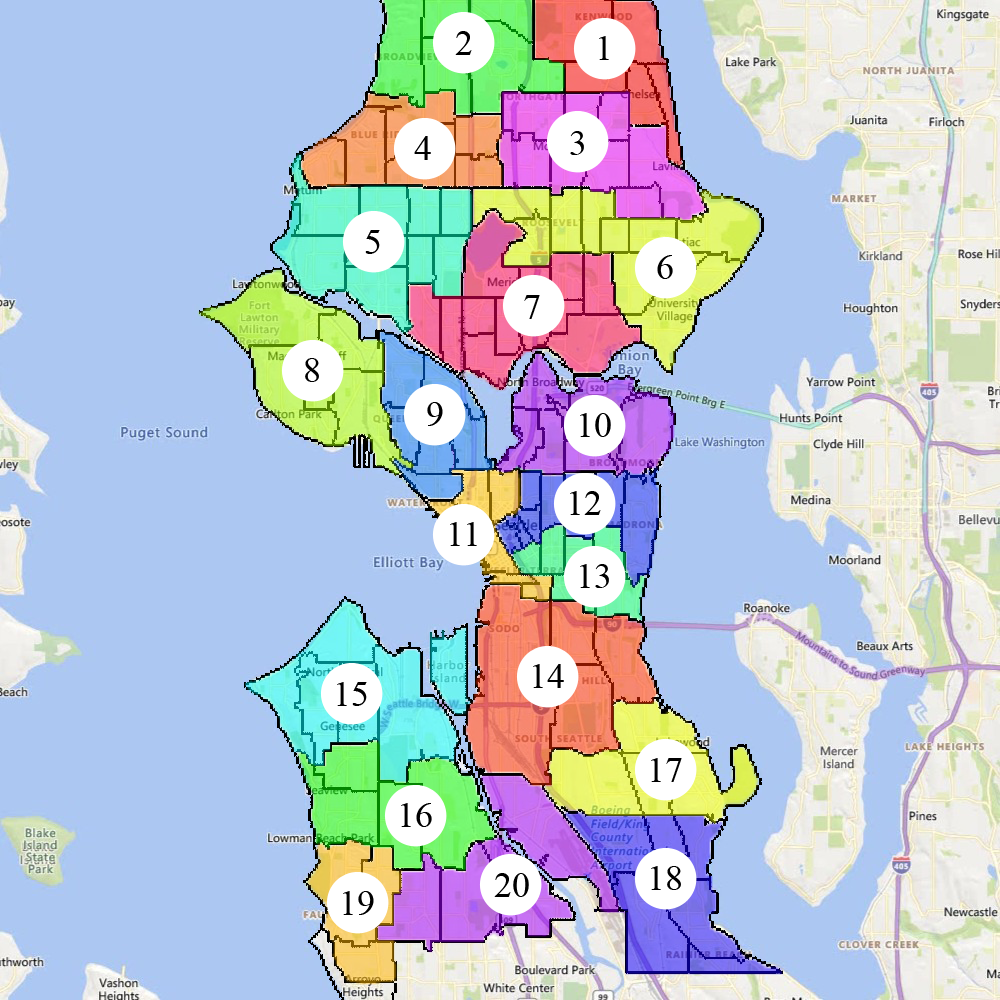}
\caption[Seattle case study areas]{Seattle case study areas and census tracts \cite{seattleCensusTractMap2010}.}
\label{fig:seattle20area}
\end{figure}

\subsection{Convex projection}

In all of our experiments, we have used the general-purpose Gurobi optimization solver \cite{gurobi} to compute the \acrshort{map} states and thus to validate the two-mode assumption. (We have also experimented with \cite{cvxpy}, but found it performing slower than Gurobi, at least over the relatively small samples considered in this proof of principles study. In the future, we plan to use existing, or developing new, \acrshort{lp} solvers designed specifically for finding the \acrshort{map} state of the attractive Ising model.) To illustrate our Prevention strategy, we took the Seattle data described above, and fed it as an input into the optimization (\ref{eq:prevention-Q-two-mode}), describing projection to the Safety Polytope, where $C(\cdot;\cdot)$ is substituted by the $l_1$ norm. CVXPY solver was used for this convex optimization task. Our code (python within jupyter notebook) is available at \url{https://github.com/mkrechetov/IsingMitigation}. 

Table \ref{tbl:projection} shows  results of our Prevention experiments on the Seattle data. We analyze $l_1$ projection to $\widehat{SP}(\Upsilon)$ where $\Upsilon$ consists of all the initial infection patterns consisting of up to $k$ nodes. In all of our experiments, the values of the field vector $h$ (uniform across the system) was fixed to $-1$. We observe that the number of constraints grows exponentially with $k$; however, the cost of intervention remains roughly the same. We intend to analyze the results of this and other (more realistic) experiments in future publications aimed at epidemiology experts and public health officials. 

\begin{table}[!t]

\centering
\begin{tabular}{l|c|c|c}
k & LP Constraints & Runtime & Cost\\
\hline\hline
1 & 801 &  1.65s   &   41.69  \\
\hline
2 & 991 &  3.04s   &  43.62  \\
\hline
3 &  2131 &  10.90s   &   44.30  \\
\hline
4 &  6976 &  100.08s  &   44.56  \\
\hline
\end{tabular}
\caption[Summary of our pandemic prevention experiments on the Seattle data]{Summary of our prevention experiments on the Seattle data. $k$, in the first column, is the maximal number of nodes in the initially infected patterns (all accounted for to construct the $k$-Safe Polytope). The second column shows number of linear constraints characterizing the $k$-Safe Polytope. Respective Run Time and Cost are shown in the 3rd and 4th column,  where Cost shows the difference in $l_1$ norm between the $(J^{(0)},h^{(0)})$, characterizing stressed but unmitigated regime, and the optimal prevention regime, resulting in $(\widehat{J}^{(\text{\tiny corr})},\widehat{h}^{(\text{\tiny corr})})$ computed according to Eq.~(\ref{eq:prevention-Q-two-mode}).}
\label{tbl:projection}
\end{table}

\section{Conclusions and Path Forward} \label{sec:Conclusions}

In this manuscript, written specifically for the AI community, we follow our prior work \cite{GMPandemic}, aimed at a broader interdisciplinary community, and explain respective inference (prediction) and control (prevention) questions/challenges. We use the language of \acrshort{gm}s,  which is one powerful tool in the modern arsenal of AI, and state the Prediction Challenge as a \acrshort{map} optimization over an attractive Ising model, which can be expressed generically as a solution of a tractable \acrfull{lp}. We then turn to the analysis of the prevention problem, which is set if the aforementioned prediction solution suggests that the probability of significant infection is above a pre-defined (by the public health experts) tolerance threshold.  We show that in its simplest formulation, the prevention problem is equivalent to finding minimal $l_1$ projection to the safety polytope, where the latter is defined by solving the aforementioned prediction problem.  In general, the polytope does not allow a description non-exponential in the size of the system. However, we suggested an approximation that allows to approximate the safety polytope efficiently - that is, linearly in the number of the initial infection patterns. The approximation is justified (empirically,  with supporting theoretical arguments,  however not yet backed by a mathematically rigorous theory) in the case when the interaction graph of the system (e.g., related to the system/city transportation and human-to-human interaction network) is sufficiently dense. We conclude by providing a quasi-realistic experimental demonstration on the \acrshort{gm} of Seattle. 

We conclude the manuscript with an incomplete list of AI challenges, presented in the order of importance (subjective), which need to be resolved to make the powerful \acrshort{gm} approach to pandemic prediction and prevention practical: 
\begin{itemize}
    \item Build a hierarchy of Probabilistic Graphical Models which allow more accurate (than Ising model) representation of the infection patterns over geographical and community graphs. The models may be both of the static (like Ising) or dynamic (like Independent Cascade Model) types. Extend the notion of the Safety Region (polytope) to the new \acrshort{gm} of pandemics.
    
    \item Consider the case when the resolution of the Prediction Challenge problem returns a positive answer - most likely future state of the system is safe,  and then develop the methodology which allows estimating the probability of crossing the safety boundary. In other words, we envision formulating and solving in the context of the \acrshort{gm} a problem which is akin to the one addressed in \cite{2019Owen,lukashevich2021importance,lukashevich2021power}: estimate the probability of finding the system outside of the Safety Polytope. 
    
    \item Construct other (than two-mode) approximations to the Safety Polytope. Approximations built on sampling of the boundaries of the safety polytope and learning (possibly reinforcement learning) are needed. 
    
    \item Develop the asymptotic (thermodynamic limit) theory which allows validating (and/or correcting systematically) the efficient (two-mode and other) approximations of the Safety Polytope. 
\end{itemize}

\phantomsection 
\addcontentsline{toc}{chapter}{Bibliography} 
\begin{singlespace}
\bibliographystyle{plainnat}
\bibliography{biblio/mypapers, biblio/schatten, biblio/epsdp, biblio/mitigation}

\begin{thebibliography}{116}
\providecommand{\natexlab}[1]{#1}
\providecommand{\url}[1]{\texttt{#1}}
\expandafter\ifx\csname urlstyle\endcsname\relax
  \providecommand{\doi}[1]{doi: #1}\else
  \providecommand{\doi}{doi: \begingroup \urlstyle{rm}\Url}\fi

\bibitem[Amini et~al.(2022)Amini, Feofanov, Pauletto, Devijver, and
  Maximov]{amini2022self}
Massih-Reza Amini, Vasilii Feofanov, Loic Pauletto, Emilie Devijver, and Yury
  Maximov.
\newblock Self-training: A survey.
\newblock \emph{arXiv preprint arXiv:2202.12040}, 2022.

\bibitem[Anderson and May(1991)]{1991Anderson}
R.M. Anderson and R.M. May.
\newblock \emph{Infectious Disease of Humans: Dynamics and Control}.
\newblock Oxford University Press, Oxford, 1991.

\bibitem[Anikin et~al.(2020)Anikin, Gasnikov, Gornov, Kamzolov, Maximov, and
  Nesterov]{anikin2020efficient}
Anton Anikin, Alexander Gasnikov, Alexander Gornov, Dmitry Kamzolov, Yury
  Maximov, and Yurii Nesterov.
\newblock Efficient numerical methods to solve sparse linear equations with
  application to pagerank.
\newblock \emph{Optimization Methods and Software}, pages 1--29, 2020.

\bibitem[Arora and Barak(2009)]{arora2009computational}
Sanjeev Arora and Boaz Barak.
\newblock \emph{Computational complexity: a modern approach}.
\newblock Cambridge University Press, 2009.

\bibitem[Barahona(1982)]{Barahona_1982}
F.~Barahona.
\newblock On the computational complexity of ising spin glass models.
\newblock \emph{Journal of Physics A: Mathematical and General}, 15\penalty0
  (10):\penalty0 3241--3253, oct 1982.
\newblock \doi{10.1088/0305-4470/15/10/028}.
\newblock URL \url{https://doi.org/10.1088/0305-4470/15/10/028}.

\bibitem[Barahona(1988)]{Barahona}
F.~Barahona.
\newblock An application of combinatorial optimization to statistical physics
  and circuit layout design.
\newblock \emph{Operations Research}, pages 493--513, 1988.

\bibitem[Barvinok(1995)]{Barvinok1995}
A.~Barvinok.
\newblock Problems of distance geometry and convex properties of quadratic
  maps.
\newblock \emph{Discrete {\&} Computational Geometry}, 13\penalty0
  (2):\penalty0 189--202, 1995.
\newblock ISSN 1432-0444.
\newblock \doi{10.1007/BF02574037}.

\bibitem[Beasley(1998)]{Beasley1998}
J.~Beasley.
\newblock {Heuristic algorithms for the unconstrained binary quadratic
  programming problem}.
\newblock Technical Report December, Management School, Imperial College,
  London, UK, 1998.

\bibitem[Boumal et~al.(2016)Boumal, Voroninski, and Bandeira]{boumal2016non}
N.~Boumal, V.~Voroninski, and A.~Bandeira.
\newblock The non-convex {Burer-Monteiro} approach works on smooth semidefinite
  programs.
\newblock In \emph{Advances in Neural Information Processing Systems}, pages
  2757--2765, 2016.

\bibitem[Boyd et~al.(2004)Boyd, Boyd, and Vandenberghe]{boyd2004convex}
Stephen Boyd, Stephen~P Boyd, and Lieven Vandenberghe.
\newblock \emph{Convex optimization}.
\newblock Cambridge university press, 2004.

\bibitem[Burer and Monteiro(2003)]{BurerMonteiro}
S.~Burer and R.~Monteiro.
\newblock A nonlinear programming algorithm for solving semidefinite programs
  via low-rank factorization.
\newblock \emph{Mathematical Programming}, 95:\penalty0 329--357, 2003.

\bibitem[Burer and Monteiro(2005)]{burer2005local}
S.~Burer and R.~Monteiro.
\newblock Local minima and convergence in low-rank semidefinite programming.
\newblock \emph{Mathematical Programming}, 103\penalty0 (3):\penalty0 427--444,
  2005.

\bibitem[Chang et~al.(2020)Chang, Pierson, Koh, Gerardin, Redbird, Grusky, and
  Leskovec]{chang_mobility_2020}
S.~Chang, E.~Pierson, P.~Koh, J.~Gerardin, B.~Redbird, D.~Grusky, and
  J.~Leskovec.
\newblock Mobility network models of {COVID}-19 explain inequities and inform
  reopening.
\newblock \emph{Nature}, November 2020.
\newblock ISSN 1476-4687.
\newblock \doi{10.1038/s41586-020-2923-3}.
\newblock URL \url{https://doi.org/10.1038/s41586-020-2923-3}.

\bibitem[Chao et~al.(2010)Chao, Halloran, Obenchain, and
  Longini~J.]{chao2010flute}
D.~Chao, M.~Halloran, V.~Obenchain, and Ira~M. Longini~J.
\newblock Flute, a publicly available stochastic influenza epidemic simulation
  model.
\newblock \emph{PLoS Comput Biol}, 6\penalty0 (1):\penalty0 e1000656, 2010.

\bibitem[Chen et~al.(2020)Chen, Lu, Chang, and Liu]{Chen_2020}
Y.C. Chen, P.E. Lu, C.S. Chang, and T.H. Liu.
\newblock A time-dependent sir model for covid-19 with undetectable infected
  persons.
\newblock \emph{IEEE Transactions on Network Science and Engineering},
  7\penalty0 (4):\penalty0 3279–3294, Oct 2020.
\newblock ISSN 2334-329X.
\newblock \doi{10.1109/tnse.2020.3024723}.
\newblock URL \url{http://dx.doi.org/10.1109/TNSE.2020.3024723}.

\bibitem[Chertkov et~al.(2021)Chertkov, Abrams, Esmaieeli~Sikaroudi, Krechetov,
  Slagle, Efrat, Fulek, and Oren]{GMPandemic}
M.~Chertkov, R.~Abrams, A.~M. Esmaieeli~Sikaroudi, M.~Krechetov, CNP Slagle,
  A.~Efrat, R.~Fulek, and E.~Oren.
\newblock Graphical models of pandemic.
\newblock
  \emph{\url{https://www.medrxiv.org/content/10.1101/2021.02.24.21252390v1.full}},
  2021.

\bibitem[Chertkov et~al.(2020)Chertkov, Chernyak, and
  Maximov]{chertkov2020gauges}
Michael Chertkov, Vladimir Chernyak, and Yury Maximov.
\newblock Gauges, loops, and polynomials for partition functions of graphical
  models.
\newblock \emph{Journal of Statistical Mechanics: Theory and Experiment},
  2020\penalty0 (12):\penalty0 124006, 2020.

\bibitem[Cormen(2009)]{cormen2009introduction}
T.~Cormen.
\newblock \emph{Introduction to algorithms}.
\newblock MIT press, 2009.

\bibitem[{CVXPY}(2021)]{cvxpy}
{CVXPY}.
\newblock Convex {O}ptimization for {E}veryone.
\newblock \emph{\url{https://www.cvxpy.org/}}, 2021.

\bibitem[Cygan et~al.(2015)Cygan, Fomin, Kowalik, Lokshtanov, Marx, Pilipczuk,
  Pilipczuk, and Saurabh]{cygan2015parameterized}
Marek Cygan, Fedor~V Fomin, {\L}ukasz Kowalik, Daniel Lokshtanov, D{\'a}niel
  Marx, Marcin Pilipczuk, Micha{\l} Pilipczuk, and Saket Saurabh.
\newblock \emph{Parameterized algorithms}, volume~5.
\newblock Springer, 2015.

\bibitem[Deza and Laurent(2009)]{Deza}
M.~Deza and M.~Laurent.
\newblock \emph{Geometry of cuts and metrics.}
\newblock Springer, 2009.

\bibitem[Downey(2018)]{2018Downey}
A.~Downey.
\newblock \emph{Think Complexity: Complexity Science and Computational
  Modeling}.
\newblock O'Reilly Media, Inc., 2nd edition, 2018.
\newblock ISBN 1549761749.

\bibitem[Erdogdu et~al.(2017)Erdogdu, Deshpande, and
  Montanari]{erdogdu2017inference}
M.~Erdogdu, Y.~Deshpande, and A.~Montanari.
\newblock Inference in graphical models via semidefinite programming
  hierarchies.
\newblock In \emph{Advances in Neural Information Processing Systems}, pages
  416--424, 2017.

\bibitem[Erdogdu et~al.(2021)Erdogdu, Ozdaglar, Parrilo, and
  Vanli]{erdogdu2018convergence}
M.~Erdogdu, A.~Ozdaglar, P.~Parrilo, and N.~Vanli.
\newblock Convergence rate of block-coordinate maximization {Burer-Monteiro}
  method for solving large sdps.
\newblock \emph{Mathematical Programming}, pages 1--39, 2021.

\bibitem[Eubank et~al.(2004)Eubank, Guclu, Anil~Kumar, Marathe, Srinivasan,
  Toroczkai, and Wang]{eubank_modelling_2004}
S.~Eubank, H.~Guclu, V.S. Anil~Kumar, M.~Marathe, A.~Srinivasan, Z.~Toroczkai,
  and N.~Wang.
\newblock Modelling disease outbreaks in realistic urban social networks.
\newblock \emph{Nature}, 429\penalty0 (6988):\penalty0 180--184, May 2004.
\newblock ISSN 1476-4687.
\newblock \doi{10.1038/nature02541}.
\newblock URL \url{https://doi.org/10.1038/nature02541}.

\bibitem[Eubank et~al.(2020)Eubank, Eckstrand, Lewis, Venkatramanan, Marathe,
  and Barrett]{eubank_commentary_2020}
S.~Eubank, I.~Eckstrand, B.~Lewis, S.~Venkatramanan, M.~Marathe, and C.L.
  Barrett.
\newblock Commentary on ferguson, et al., “impact of non-pharmaceutical
  interventions ({NPIs}) to reduce {COVID}-19 mortality and healthcare
  demand”.
\newblock \emph{Bulletin of Mathematical Biology}, 82\penalty0 (4):\penalty0
  52, 2020.
\newblock ISSN 1522-9602.
\newblock \doi{10.1007/s11538-020-00726-x}.
\newblock URL \url{https://doi.org/10.1007/s11538-020-00726-x}.

\bibitem[Fawzi et~al.(2019)Fawzi, Saunderson, and Parrilo]{Fawzi2019}
H.~Fawzi, J.~Saunderson, and P.~Parrilo.
\newblock Semidefinite approximations of the matrix logarithm.
\newblock \emph{Foundations of Computational Mathematics}, 19\penalty0
  (2):\penalty0 259--296, Apr 2019.
\newblock ISSN 1615-3383.
\newblock \doi{10.1007/s10208-018-9385-0}.
\newblock URL \url{https://doi.org/10.1007/s10208-018-9385-0}.

\bibitem[Fazel et~al.(2003)Fazel, Hindi, and Boyd]{Fazel2003}
M.~Fazel, H.~Hindi, and S.~Boyd.
\newblock {Log-det heuristic for matrix rank minimization with applications to
  Hankel and Euclidean distance matrices}.
\newblock In \emph{Proceedings of the 2003 American Control Conference, 2003.},
  volume~3, pages 2156--2162, 2003.
\newblock ISBN 0-7803-7896-2.
\newblock \doi{10.1109/ACC.2003.1243393}.

\bibitem[Ferguson et~al.(2005)Ferguson, Cummings, Cauchemez, Fraser, Riley,
  Meeyai, Iamsirithaworn, and Burke]{ferguson_strategies_2005}
N.~Ferguson, D.~Cummings, S.~Cauchemez, C.~Fraser, S.~Riley, A.~Meeyai,
  S.~Iamsirithaworn, and D.~Burke.
\newblock Strategies for containing an emerging influenza pandemic in
  {Southeast} {Asia}.
\newblock \emph{Nature}, 437\penalty0 (7056):\penalty0 209--214, September
  2005.
\newblock ISSN 1476-4687.
\newblock \doi{10.1038/nature04017}.
\newblock URL \url{https://doi.org/10.1038/nature04017}.

\bibitem[Ferguson et~al.(2006)Ferguson, Cummings, Fraser, Cajka, Cooley, and
  Burke]{ferguson_strategies_2006}
N.~Ferguson, D.~Cummings, C.~Fraser, J.C. Cajka, P.~Cooley, and D.~Burke.
\newblock Strategies for mitigating an influenza pandemic.
\newblock \emph{Nature}, 442\penalty0 (7101):\penalty0 448--452, July 2006.
\newblock ISSN 1476-4687.
\newblock \doi{10.1038/nature04795}.
\newblock URL \url{https://doi.org/10.1038/nature04795}.

\bibitem[Ferguson et~al.(2020)Ferguson, Laydon, Nedjati-Gilani, Imai, Ainslie,
  Baguelin, Bhatia, Boonyasiri, Cucunubá, Cuomo-Dannenburg, Dighe, Dorigatti,
  Fu, Gaythorpe, Green, Hamlet, Hinsley, Okell, van Elsland, and
  Ghani]{2020Ferguson}
N.~Ferguson, D.~Laydon, G.~Nedjati-Gilani, N.~Imai, K.~Ainslie, M.~Baguelin,
  S.~Bhatia, A.~Boonyasiri, Z.~Cucunubá, G.~Cuomo-Dannenburg, A.~Dighe,
  I.~Dorigatti, H.~Fu, K.~Gaythorpe, W.~Green, A.~Hamlet, W.~Hinsley, L.~Okell,
  S.~van Elsland, and A.~Ghani.
\newblock Report 9: Impact of non-pharmaceutical interventions (npis) to reduce
  covid-19 mortality and healthcare demand.
\newblock
  \emph{\url{https://www.imperial.ac.uk/media/imperial-college/medicine/sph/ide/gida-fellowships/Imperial-College-COVID19-NPI-modelling-16-03-2020.pdf}},
  2020.

\bibitem[Frostig et~al.(2014)Frostig, Wang, Liang, and Manning]{lowrankMAP}
R.~Frostig, S.~Wang, P.~Liang, and C.~Manning.
\newblock Simple {MAP} inference via low-rank relaxations.
\newblock In \emph{Advances in Neural Information Processing Systems}, pages
  3077--3085, 2014.

\bibitem[Garey and Johnson(2002)]{garey2002computers}
M.~Garey and D.~Johnson.
\newblock \emph{Computers and intractability}, volume~29.
\newblock wh freeman New York, 2002.

\bibitem[Germann et~al.(2006)Germann, Kadau, Longini, and Macken]{2006Germann}
Timothy~C. Germann, Kai Kadau, Ira~M. Longini, and Catherine~A. Macken.
\newblock Mitigation strategies for pandemic influenza in the united states.
\newblock \emph{Proceedings of the National Academy of Sciences}, 103\penalty0
  (15):\penalty0 5935--5940, 2006.
\newblock \doi{10.1073/pnas.0601266103}.
\newblock URL \url{https://www.pnas.org/content/103/15/5935}.

\bibitem[Ghaddar et~al.(2016)Ghaddar, Marecek, and Mevissen]{7024950}
B.~Ghaddar, J.~Marecek, and M.~Mevissen.
\newblock Optimal power flow as a polynomial optimization problem.
\newblock \emph{IEEE Transactions on Power Systems}, 31\penalty0 (1):\penalty0
  539--546, Jan 2016.
\newblock ISSN 0885-8950.
\newblock \doi{10.1109/TPWRS.2015.2390037}.

\bibitem[Gill et~al.(2019)Gill, Murray, and Wright]{gill2019practical}
Philip~E Gill, Walter Murray, and Margaret~H Wright.
\newblock \emph{Practical optimization}.
\newblock SIAM, 2019.

\bibitem[Goemans and Williamson(1995)]{goemans1995improved}
M.~Goemans and D.~Williamson.
\newblock Improved approximation algorithms for maximum cut and satisfiability
  problems using semidefinite programming.
\newblock \emph{Journal of the ACM}, 42\penalty0 (6):\penalty0 1115--1145,
  1995.

\bibitem[Gomez-Rodriguez et~al.(2012)Gomez-Rodriguez, Leskovec, and
  Krause]{2012GomezLeskovecKrause}
M.~Gomez-Rodriguez, J.~Leskovec, and A.~Krause.
\newblock Inferring networks of diffusion and influence.
\newblock \emph{ACM Trans. Knowl. Discov. Data}, 5\penalty0 (4), February 2012.
\newblock ISSN 1556-4681.
\newblock \doi{10.1145/2086737.2086741}.
\newblock URL \url{https://doi.org/10.1145/2086737.2086741}.

\bibitem[Grant and Boyd(2008)]{gb08}
M.~Grant and S.~Boyd.
\newblock Graph implementations for nonsmooth convex programs.
\newblock In \emph{Recent Advances in Learning and Control}, Lecture Notes in
  Control and Information Sciences, pages 95--110. Springer-Verlag Limited,
  2008.

\bibitem[Grant and Boyd(2014)]{cvx}
M.~Grant and S.~Boyd.
\newblock {CVX}: Matlab software for disciplined convex programming, version
  2.1.
\newblock \url{http://cvxr.com/cvx}, March 2014.

\bibitem[{Gurobi Optimization, LLC}(2021)]{gurobi}
{Gurobi Optimization, LLC}.
\newblock {Gurobi Optimizer Reference Manual}.
\newblock \emph{\url{https://www.gurobi.com}}, 2021.

\bibitem[Halloran et~al.(2008)Halloran, Ferguson, Eubank, Longini, Cummings,
  Lewis, Xu, Fraser, Vullikanti, Germann, Wagener, Beckman, Kadau, Barrett,
  Macken, Burke, and Cooley]{2008Halloran}
M.~Halloran, N.~Ferguson, S.~Eubank, I.~Longini, D.~Cummings, B.~Lewis, S.~Xu,
  C.~Fraser, A.~Vullikanti, T.~Germann, D.~Wagener, R.~Beckman, K.~Kadau,
  C.~Barrett, C.~Macken, D.~Burke, and P.~Cooley.
\newblock Modeling targeted layered containment of an influenza pandemic in the
  united states.
\newblock \emph{Proceedings of the National Academy of Sciences}, 105\penalty0
  (12):\penalty0 4639--4644, 2008.
\newblock ISSN 0027-8424.
\newblock \doi{10.1073/pnas.0706849105}.
\newblock URL \url{https://www.pnas.org/content/105/12/4639}.

\bibitem[Helmberg and Rendl(1997)]{Helmberg1997}
C.~Helmberg and F.~Rendl.
\newblock {A Spectral Bundle Method for Semidefinite Programming}.
\newblock \emph{SIAM Journal on Optimization}, 10\penalty0 (August):\penalty0
  673--696, 1997.
\newblock ISSN 10526234.
\newblock \doi{10.1137/S1052623497328987}.

\bibitem[Helmberg et~al.(1995)Helmberg, Poljak, and Rendl]{Helmberg1995}
C.~Helmberg, S.~Poljak, and F.~Rendl.
\newblock {Combining semidefinite and polyhedral relaxations for integer
  programs}.
\newblock \emph{Integer Programming and Combinatorial Optimization}, pages
  124--134, 1995.
\newblock ISSN 16113349.
\newblock URL \url{http://www.springerlink.com/index/F377K847M974R422.pdf}.

\bibitem[Hethcote(2000)]{2000Hethcore}
H.W. Hethcote.
\newblock The mathematics of infectious diseases.
\newblock \emph{SIAM Rev.}, 42\penalty0 (4):\penalty0 599–653, December 2000.
\newblock ISSN 0036-1445.
\newblock \doi{10.1137/S0036144500371907}.
\newblock URL \url{https://doi.org/10.1137/S0036144500371907}.

\bibitem[Holmes et~al.(2007)Holmes, Gray, and Isbell]{Holmes}
M.~Holmes, A.~Gray, and C.~Isbell.
\newblock Fast {SVD} for large-scale matrices.
\newblock \emph{Workshop on Efficient Machine Learning at NIPS}, 58:\penalty0
  249--252, 2007.

\bibitem[Hu et~al.(2013)Hu, Zhang, Ye, Li, and He]{HuTruncated}
Y.~Hu, D.~Zhang, J.~Ye, X.~Li, and X.~He.
\newblock Fast and accurate matrix completion via truncated nuclear norm
  regularization.
\newblock \emph{IEEE Transactions on Pattern Analysis and Machine
  Intelligence}, pages 2117--2130, 2013.

\bibitem[Karp(1972)]{Karp1972}
R.~Karp.
\newblock \emph{Reducibility among Combinatorial Problems}, pages 85--103.
\newblock Springer US, Boston, MA, 1972.
\newblock ISBN 978-1-4684-2001-2.
\newblock \doi{10.1007/978-1-4684-2001-2_9}.

\bibitem[Karp(1975)]{karp1975computational}
Richard~M Karp.
\newblock On the computational complexity of combinatorial problems.
\newblock \emph{Networks}, 5\penalty0 (1):\penalty0 45--68, 1975.

\bibitem[Kaxiras and Neofotistos(2020)]{kaxiras2020multiple}
E.~Kaxiras and G.~Neofotistos.
\newblock Multiple epidemic wave model of the covid-19 pandemic: Modeling
  study.
\newblock \emph{Journal of Medical Internet Research}, 22\penalty0
  (7):\penalty0 e20912, 2020.

\bibitem[Kempe et~al.(2003)Kempe, Kleinberg, and
  Tardos]{2003KempeKleinbergTardos}
D.~Kempe, J.~Kleinberg, and E.~Tardos.
\newblock Maximizing the spread of influence through a social network.
\newblock In \emph{Proceedings of the Ninth ACM SIGKDD International Conference
  on Knowledge Discovery and Data Mining}, KDD ’03, page 137–146, New York,
  NY, USA, 2003. Association for Computing Machinery.
\newblock ISBN 1581137370.
\newblock \doi{10.1145/956750.956769}.
\newblock URL \url{https://doi.org/10.1145/956750.956769}.

\bibitem[Kermack et~al.(1927)Kermack, McKendrick, and Walker]{1927Kermack}
W.~Kermack, A.G. McKendrick, and G.~Walker.
\newblock A contribution to the mathematical theory of epidemics.
\newblock \emph{Proceedings of the Royal Society of London. Series A,
  Containing Papers of a Mathematical and Physical Character}, 115\penalty0
  (772):\penalty0 700--721, 1927.
\newblock \doi{10.1098/rspa.1927.0118}.

\bibitem[Kerr et~al.(2021)Kerr, Stuart, Mistry, Abeysuriya, Rosenfeld, Hart,
  Nunez, Cohen, Selvaraj, Hagedorn, George, Jastrzebski, Izzo, Fowler, Palmer,
  Delport, Scott, Kelly, Bennette, Wagner, Chang, Oron, Wenger,
  Panovska-Griffiths, Famulare, and Klein]{ABM-Gates}
C.~Kerr, R.~Stuart, D.~Mistry, R.G. Abeysuriya, K.~Rosenfeld, G.~Hart,
  R.~Nunez, J.~Cohen, P.~Selvaraj, B.~Hagedorn, L.~George, M.~Jastrzebski,
  A.~Izzo, G.~Fowler, A.~Palmer, D.~Delport, N.~Scott, S.~Kelly, C.~Bennette,
  B.~Wagner, S.~Chang, A.~Oron, E.~Wenger, J.~Panovska-Griffiths, M.~Famulare,
  and D.~Klein.
\newblock Covasim: An agent-based model of covid-19 dynamics and interventions.
\newblock \emph{PLOS Computational Biology}, 17\penalty0 (7):\penalty0 1--32,
  07 2021.
\newblock \doi{10.1371/journal.pcbi.1009149}.
\newblock URL \url{https://doi.org/10.1371/journal.pcbi.1009149}.

\bibitem[Keshavan et~al.(2010)Keshavan, Montanari, and Oh]{keshavan2010matrix}
R.~Keshavan, A.~Montanari, and S.~Oh.
\newblock Matrix completion from a few entries.
\newblock \emph{IEEE Transactions on Information Theory}, 56\penalty0
  (6):\penalty0 2980--2998, 2010.

\bibitem[Khalil et~al.(2013)Khalil, Dilkina, and Song]{KhaDilSon13}
E.~Khalil, B.~Dilkina, and L.~Song.
\newblock Cuttingedge: Influence minimization in networks.
\newblock In \emph{Workshop on Frontiers of Network Analysis: Methods, Models,
  and Applications at NIPS}, 2013.
\newblock URL \url{files/papers/CuttingEdge.pdf}.

\bibitem[Khot(2002)]{Khot:2002:PUG:509907.510017}
S.~Khot.
\newblock On the power of unique 2-prover 1-round games.
\newblock In \emph{Proceedings of the Thiry-fourth Annual ACM Symposium on
  Theory of Computing}, STOC '02, pages 767--775, New York, NY, USA, 2002. ACM.
\newblock ISBN 1-58113-495-9.
\newblock \doi{10.1145/509907.510017}.

\bibitem[Kleinhans et~al.(1991)Kleinhans, Sigl, Johannes, and
  Antreich]{kleinhans1991gordian}
J.~Kleinhans, G.~Sigl, F.~Johannes, and K.~Antreich.
\newblock Gordian: Vlsi placement by quadratic programming and slicing
  optimization.
\newblock \emph{IEEE Transactions on Computer-Aided Design of Integrated
  Circuits and Systems}, 10\penalty0 (3):\penalty0 356--365, 1991.

\bibitem[Kochenberger(2014)]{UBQP}
G.~Kochenberger.
\newblock The unconstrained binary quadratic programming problem: a survey.
\newblock \emph{Journal of Combinatorial Optimization}, pages 58--81, 2014.

\bibitem[Krechetov et~al.(2019)Krechetov, Maximov, Mareček, and
  Takáč]{EPSDP}
M.~Krechetov, Y.~Maximov, J.~Mareček, and M.~Takáč.
\newblock Entropy-penalized semidefinite programming.
\newblock \emph{IJCAI International Joint Conference on Artificial
  Intelligence}, pages 1123--1129, 2019.

\bibitem[Krechetov et~al.(2021)Krechetov, Sikaroudi, Efrat, Polishchuk, and
  Chertkov]{GMMitigation}
M.~Krechetov, A.~M.~E. Sikaroudi, A.~Efrat, V.~Polishchuk, and M.~Chertkov.
\newblock Prediction and prevention of pandemics via graphical model inference
  and convex programming.
\newblock \emph{Scientific Reports, 12(1)}, pages 1--11, 2021.

\bibitem[Krislock et~al.(2014)Krislock, Malick, and
  Roupin]{krislock2014improved}
N.~Krislock, J.~Malick, and F.~Roupin.
\newblock Improved semidefinite bounding procedure for solving max-cut problems
  to optimality.
\newblock \emph{Mathematical Programming}, 143\penalty0 (1-2):\penalty0 61--86,
  2014.

\bibitem[{LANL}(2020)]{lanl2020covid}
{LANL}.
\newblock Covid-19 confirmed and forecasted case data.
\newblock \emph{\url{https://covid-19.bsvgateway.org/}}, 2020.

\bibitem[Lasserre(2015)]{lasserre2015introduction}
J.B. Lasserre.
\newblock \emph{An introduction to polynomial and semi-algebraic optimization},
  volume~52.
\newblock Cambridge University Press, 2015.

\bibitem[Lemon et~al.(2016)Lemon, So, and Ye]{Lemon}
A.~Lemon, A.~So, and Y.~Ye.
\newblock \emph{Low-rank semidefinite programming: Theory and applications.}
\newblock Foundations and Trends® in Optimization, 2016.

\bibitem[Li(2014)]{Li}
C.~Li.
\newblock {Approximation of Matrix Rank Function and Its Application to Matrix
  Rank Minimization}.
\newblock \emph{Journal of Optimization Theory and Applications}, pages
  569--594, 2014.

\bibitem[Liers et~al.(2004)Liers, Jünger, Reinelt, and Rinaldi]{MCising}
F.~Liers, M.~Jünger, G.~Reinelt, and G.~Rinaldi.
\newblock Computing exact ground states of hard ising spin glass problems by
  branch-and-cut.
\newblock \emph{New optimization algorithms in physics}, pages 47--69, 2004.

\bibitem[Likhosherstov et~al.(2019{\natexlab{a}})Likhosherstov, Maximov, and
  Chertkov]{likhosherstov2019new}
Valerii Likhosherstov, Yury Maximov, and Michael Chertkov.
\newblock A new family of tractable ising models.
\newblock \emph{arXiv preprint arXiv:1906.06431}, 2019{\natexlab{a}}.

\bibitem[Likhosherstov et~al.(2019{\natexlab{b}})Likhosherstov, Maximov, and
  Chertkov]{likhosherstov2019inference}
Valerii Likhosherstov, Yury Maximov, and Misha Chertkov.
\newblock Inference and sampling of $ k\_33 $-free ising models.
\newblock In \emph{International Conference on Machine Learning}, pages
  3963--3972. PMLR, 2019{\natexlab{b}}.

\bibitem[Likhosherstov et~al.(2020)Likhosherstov, Maximov, and
  Chertkov]{likhosherstov2020tractable}
Valerii Likhosherstov, Yury Maximov, and Michael Chertkov.
\newblock Tractable minor-free generalization of planar zero-field ising
  models.
\newblock \emph{Journal of Statistical Mechanics: Theory and Experiment},
  2020\penalty0 (12):\penalty0 124007, 2020.

\bibitem[Longini et~al.(2005)Longini, Nizam, Xu, Ungchusak, Hanshaoworakul,
  Cummings, and Halloran]{2005Longini}
{I.} Longini, A.~Nizam, S.~Xu, K.~Ungchusak, W.~Hanshaoworakul, {D.} Cummings,
  and {M.} Halloran.
\newblock Containing pandemic influenza at the source.
\newblock \emph{Science}, 309\penalty0 (5737):\penalty0 1083--1087, August
  2005.
\newblock ISSN 0036-8075.
\newblock \doi{10.1126/science.1115717}.

\bibitem[Lovasi and et.al.(2020)]{ABM-Columbia}
G.~Lovasi and et.al.
\newblock Population health methods: Agent based modeling, 2020.
\newblock URL
  \url{https://www.publichealth.columbia.edu/research/population-health-methods/agent-based-modeling}.

\bibitem[Luchnikov et~al.(2021)Luchnikov, Krechetov, and Filippov]{Luchnikov}
I.~Luchnikov, M.~Krechetov, and S.~Filippov.
\newblock Riemannian geometry and automatic differentiation for optimization
  problems of quantum physics and quantum technologies.
\newblock \emph{New Journal of Physics}, 2021.

\bibitem[Lukashevich and Maximov(2021)]{lukashevich2021power}
Aleksander Lukashevich and Yury Maximov.
\newblock Power grid reliability estimation via adaptive importance sampling.
\newblock \emph{IEEE Control Systems Letters}, 6:\penalty0 1010--1015, 2021.

\bibitem[Lukashevich et~al.(2021)Lukashevich, Gorchakov, Vorobev, Deka, and
  Maximov]{lukashevich2021importance}
Aleksander Lukashevich, Vyacheslav Gorchakov, Petr Vorobev, Deepjyoti Deka, and
  Yury Maximov.
\newblock Importance sampling approach to chance-constrained dc optimal power
  flow.
\newblock \emph{arXiv preprint arXiv:2111.11729}, 2021.

\bibitem[Madani et~al.(2014)Madani, Fazelnia, Sojoudi, and Lavaei]{Madani2014}
R.~Madani, G.~Fazelnia, S.~Sojoudi, and J.~Lavaei.
\newblock Low-rank solutions of matrix inequalities with applications to
  polynomial optimization and matrix completion problems.
\newblock In \emph{53rd IEEE Conference on Decision and Control}, pages
  4328--4335, Dec 2014.

\bibitem[Mare{\v c}ek and M.(2017)]{Marecek2017}
J.~Mare{\v c}ek and Tak{\' a}{\v c} M.
\newblock A low-rank coordinate-descent algorithm for semidefinite programming
  relaxations of optimal power flow.
\newblock \emph{Optimization Methods and Software}, 32\penalty0 (4):\penalty0
  849--871, 2017.
\newblock \doi{10.1080/10556788.2017.1288729}.

\bibitem[Maximov et~al.(2018)Maximov, Amini, and
  Harchaoui]{maximov2018rademacher}
Yury Maximov, Massih-Reza Amini, and Zaid Harchaoui.
\newblock Rademacher complexity bounds for a penalized multi-class
  semi-supervised algorithm.
\newblock \emph{Journal of Artificial Intelligence Research}, 61:\penalty0
  761--786, 2018.

\bibitem[Maziarz and Zach(2020)]{2020ABM-calibration}
M.~Maziarz and M.~Zach.
\newblock Agent-based modelling for sars-cov-2 epidemic prediction and
  intervention assessment: A methodological appraisal.
\newblock \emph{Journal of Evaluation in Clinical Practice}, 26\penalty0
  (5):\penalty0 1352--1360, 2020.
\newblock \doi{10.1111/jep.13459}.
\newblock URL \url{https://onlinelibrary.wiley.com/doi/abs/10.1111/jep.13459}.

\bibitem[Mezard and Montanari(2009)]{MezardMontanari}
M.~Mezard and A.~Montanari.
\newblock \emph{Information, physics, and computation}.
\newblock Oxford University Press., 2009.

\bibitem[Mezard et~al.(1986)Mezard, Parisi, and Virasoro]{MezardParisiVirasoro}
M.~Mezard, G.~Parisi, and M.~Virasoro.
\newblock \emph{Spin Glass Theory and Beyond}.
\newblock WORLD SCIENTIFIC, 1986.
\newblock \doi{10.1142/0271}.
\newblock URL \url{https://www.worldscientific.com/doi/abs/10.1142/0271}.

\bibitem[Mikhalev et~al.(2020)Mikhalev, Emchinov, Chevalier, Maximov, and
  Vorobev]{mikhalev2020bayesian}
Artem Mikhalev, Alexander Emchinov, Samuel Chevalier, Yury Maximov, and Petr
  Vorobev.
\newblock A bayesian framework for power system components identification.
\newblock In \emph{2020 IEEE Power \& Energy Society General Meeting (PESGM)},
  pages 1--5. IEEE, 2020.

\bibitem[Mohan and Fazel(2012)]{Mohan2012}
K.~Mohan and M.~Fazel.
\newblock Iterative reweighted algorithms for matrix rank minimization.
\newblock \emph{Journal of Machine Learning Research}, 13\penalty0
  (Nov):\penalty0 3441--3473, 2012.

\bibitem[Nesterov(1998)]{Nesterov}
Y.~Nesterov.
\newblock Semidefinite relaxation and nonconvex quadratic optimization.
\newblock \emph{Optimization Methods and Software}, 9\penalty0 (1-3):\penalty0
  141--160, 1998.
\newblock \doi{10.1080/10556789808805690}.

\bibitem[Netrapalli and Sanghavi(2012)]{2012NetrapalliSanghavi}
P.~Netrapalli and S.~Sanghavi.
\newblock Learning the graph of epidemic cascades.
\newblock In \emph{Proceedings of the 12th ACM SIGMETRICS/PERFORMANCE Joint
  International Conference on Measurement and Modeling of Computer Systems},
  SIGMETRICS ’12, page 211–222, New York, NY, USA, 2012. Association for
  Computing Machinery.
\newblock ISBN 9781450310970.
\newblock \doi{10.1145/2254756.2254783}.
\newblock URL \url{https://doi.org/10.1145/2254756.2254783}.

\bibitem[Nie et~al.(2012)Nie, Huang, and Ding]{Nie:2012}
F.~Nie, H.~Huang, and C.~Ding.
\newblock Low-rank matrix recovery via efficient schatten p-norm minimization.
\newblock In \emph{Proceedings of the Twenty-Sixth AAAI Conference on
  Artificial Intelligence}, AAAI'12, pages 655--661. AAAI Press, 2012.

\bibitem[{Office of Planning \& Community Development,
  Seattle}(2010)]{seattleCensusTractMap2010}
{Office of Planning \& Community Development, Seattle}.
\newblock Census tract map of seattle,
  \url{https://www.seattle.gov/Documents/Departments/OPCD/Demographics/GeographicFilesandMaps/2010CensusTractMap.pdf},
  2010.

\bibitem[Owen et~al.(2019)Owen, Maximov, and Chertkov]{2019Owen}
A.B. Owen, Y.~Maximov, and M.~Chertkov.
\newblock {Importance sampling the union of rare events with an application to
  power systems analysis}.
\newblock \emph{Electronic Journal of Statistics}, 13\penalty0 (1):\penalty0
  231 -- 254, 2019.
\newblock \doi{10.1214/18-EJS1527}.
\newblock URL \url{https://doi.org/10.1214/18-EJS1527}.

\bibitem[P.(1998)]{Pataki98}
Gábor P.
\newblock On the rank of extreme matrices in semidefinite programs and the
  multiplicity of optimal eigenvalues.
\newblock \emph{Mathematics of Operations Research}, 23\penalty0 (2):\penalty0
  339--358, 1998.
\newblock ISSN 0364-765X.

\bibitem[Parker and Rardin(2014)]{parker2014discrete}
R.~Parker and R.~Rardin.
\newblock \emph{Discrete optimization}.
\newblock Elsevier, 2014.

\bibitem[Parrilo(2003)]{Parrilo}
P.~Parrilo.
\newblock Semidefinite programming relaxations for semialgebraic problems.
\newblock \emph{Mathematical programming}, page 293–320, 2003.

\bibitem[Pogodin et~al.(2017)Pogodin, Krechetov, and Maximov]{Pogodin}
R.~Pogodin, M.~Krechetov, and Y.~Maximov.
\newblock Efficient rank minimization to tighten semidefinite programming for
  unconstrained binary quadratic optimization.
\newblock \emph{55th Annual Allerton Conference on Communication, Control, and
  Computing (Allerton)}, pages 1153--1159, 2017.

\bibitem[Recht et~al.(2010)Recht, Fazel, and Parrilo]{recht2010guaranteed}
B.~Recht, M.~Fazel, and P.~Parrilo.
\newblock Guaranteed minimum-rank solutions of linear matrix equations via
  nuclear norm minimization.
\newblock \emph{SIAM Review}, 52\penalty0 (3):\penalty0 471--501, 2010.

\bibitem[Ren et~al.(2014)Ren, Jiang, Yuan, and Wang]{ren2014optimizing}
J.~Ren, X.~Jiang, J.~Yuan, and G.~Wang.
\newblock Optimizing lbp structure for visual recognition using binary
  quadratic programming.
\newblock \emph{IEEE Signal Processing Letters}, 21\penalty0 (11):\penalty0
  1346--1350, 2014.

\bibitem[Rendl et~al.(2010)Rendl, Rinaldi, and Wiegele]{Rinaldi}
F.~Rendl, G.~Rinaldi, and A.~Wiegele.
\newblock Solving max-cut to optimality by intersecting semidefinite and
  polyhedral relaxations.
\newblock \emph{Mathematical Programming}, pages 307--335, 2010.

\bibitem[Richardson and Urbanke(2008)]{RichardsonUrbanke}
T.~Richardson and R.~Urbanke.
\newblock \emph{Modern Coding Theory}.
\newblock Cambridge University Press, USA, 2008.
\newblock ISBN 0521852293.

\bibitem[Rosenfeld et~al.(2016)Rosenfeld, Nitzan, and
  Globerson]{2016RosenfeldNitzanGloberson}
N.~Rosenfeld, M.~Nitzan, and A.~Globerson.
\newblock Discriminative learning of infection models.
\newblock In \emph{Proceedings of the Ninth ACM International Conference on Web
  Search and Data Mining}, WSDM ’16, page 563–572, New York, NY, USA, 2016.
  Association for Computing Machinery.
\newblock ISBN 9781450337168.
\newblock \doi{10.1145/2835776.2835802}.
\newblock URL \url{https://doi.org/10.1145/2835776.2835802}.

\bibitem[Ross(1910)]{1910Ross}
R.~Ross.
\newblock \emph{The Prevention of Malaria}.
\newblock John Murray, London, 1910.

\bibitem[Safe{G}raph(2021{\natexlab{a}})]{SafeGraph-DataConsortium}
Safe{G}raph.
\newblock Safe{G}raph {COVID}-19 {D}ata {C}onsortium. {S}an {F}rancisco, {CA}:
  {S}afe{G}raph {I}nc.,
  \url{https://www.safegraph.com/covid-19-data-consortium}, 2021{\natexlab{a}}.

\bibitem[Safe{G}raph(2021{\natexlab{b}})]{SafeGraph-Mobility}
Safe{G}raph.
\newblock Safegraph social distancing metrics. san francisco, ca: Safegraph
  inc., \url{https://docs.safegraph.com/docs/social-distancing-metrics},
  2021{\natexlab{b}}.

\bibitem[Sanjeev et~al.(2012)Sanjeev, Ge, Kannan, and
  Moitra]{arora2012computing}
A.~Sanjeev, R.~Ge, R.~Kannan, and A.~Moitra.
\newblock Computing a nonnegative matrix factorization--provably.
\newblock In \emph{Proceedings of the forty-fourth annual ACM {S}ymposium on
  Theory of {C}omputing}, pages 145--162. ACM, 2012.

\bibitem[Schraudolph and Kamenetsky(2008)]{schraudolph2008efficient}
Nicol Schraudolph and Dmitry Kamenetsky.
\newblock Efficient exact inference in planar ising models.
\newblock \emph{Advances in Neural Information Processing Systems}, 21, 2008.

\bibitem[Shang et~al.(2018)Shang, Cheng, Liu, Luo, and Lin]{Shang2018}
F.~Shang, J.~Cheng, Y.~Liu, Z.~Luo, and Z.~Lin.
\newblock Bilinear factor matrix norm minimization for robust {PCA}: Algorithms
  and applications.
\newblock \emph{IEEE Transactions on Pattern Analysis and Machine
  Intelligence}, 40\penalty0 (9):\penalty0 2066--2080, 2018.

\bibitem[Sidana et~al.(2021)Sidana, Trofimov, Horodnytskyi, Laclau, Maximov,
  and Amini]{sidana2021user}
Sumit Sidana, Mikhail Trofimov, Oleh Horodnytskyi, Charlotte Laclau, Yury
  Maximov, and Massih-Reza Amini.
\newblock User preference and embedding learning with implicit feedback for
  recommender systems.
\newblock \emph{Data Mining and Knowledge Discovery}, 35\penalty0 (2):\penalty0
  568--592, 2021.

\bibitem[So et~al.(2008)So, Ye, and Zhang]{So2008}
A.M.-C. So, Y.~Ye, and J.~Zhang.
\newblock A unified theorem on sdp rank reduction.
\newblock \emph{Mathematics of Operations Research}, 33\penalty0 (4):\penalty0
  910--920, 2008.
\newblock ISSN 0364-765X.
\newblock \doi{10.1287/moor.1080.0326}.

\bibitem[Song et~al.(2017)Song, Misiakiewicz, Montanari, and
  Oliveira]{Montanari}
M.~Song, T.~Misiakiewicz, A.~Montanari, and R.~Oliveira.
\newblock Solving {SDPs} for synchronization and {MaxCut} problems via the
  {Grothendieck} inequality.
\newblock In \emph{2017 Conference on Learning Theory}, volume~65 of
  \emph{Proceedings of Machine Learning Research}, pages 1476--1515, 07--10 Jul
  2017.

\bibitem[Srebro et~al.(2004)Srebro, Rennie, and Jaakkola]{Srebro}
N.~Srebro, J.~Rennie, and T.~Jaakkola.
\newblock Maximum-margin matrix factorization.
\newblock In \emph{Advances in Neural Information Processing Systems}, NIPS'04,
  pages 1329--1336, 2004.

\bibitem[Stulov et~al.(2020)Stulov, Sobajic, Maximov, Deka, and
  Chertkov]{stulov2020learning}
Nikolay Stulov, Dejan~J Sobajic, Yury Maximov, Deepjyoti Deka, and Michael
  Chertkov.
\newblock Learning model of generator from terminal data.
\newblock \emph{Electric Power Systems Research}, 189:\penalty0 106742, 2020.

\bibitem[Sun et~al.(2013)Sun, Xiang, and Ye]{SunTruncated}
Q.~Sun, S.~Xiang, and J.~Ye.
\newblock Robust principal component analysis via capped norms.
\newblock In \emph{Proceedings of the 19th ACM SIGKDD international conference
  on Knowledge discovery and data mining}, pages 311--319, 2013.

\bibitem[Torr(2003)]{torr2003solving}
P.~Torr.
\newblock Solving {M}arkov random fields using semidefinite programming.
\newblock In \emph{Proceedings of the Ninth International Workshop on
  Artificial Intelligence and Statistics, {AISTATS}}, pages 1--8, 2003.

\bibitem[{United States Census Bureau}(2019)]{USBureauGlossary}
{United States Census Bureau}.
\newblock United states census bureau glossary,
  \url{https://www.census.gov/programs-surveys/geography/about/glossary.html},
  2019.
\newblock URL
  \url{https://www.census.gov/programs-surveys/geography/about/glossary.html}.

\bibitem[{United States Census Bureau}(2021)]{USBureauTiger}
{United States Census Bureau}.
\newblock United states census bureau. tiger line shapefiles technical
  documentation, 2021.

\bibitem[Wainwright and Jordan(2008)]{ExpFamilies}
M.~Wainwright and M.~Jordan.
\newblock \emph{Graphical models, exponential families, and variational
  inference.}
\newblock Foundations and Trends® in Machine Learning, 2008.

\bibitem[Wang et~al.(2017)Wang, Shen, van~den Hengel, and Torr]{wang2017large}
P.~Wang, C.~Shen, A.~van~den Hengel, and P.~Torr.
\newblock Large-scale binary quadratic optimization using semidefinite
  relaxation and applications.
\newblock \emph{IEEE transactions on pattern analysis and machine
  intelligence}, 39\penalty0 (3):\penalty0 470--485, 2017.

\bibitem[Wikipedia(2020)]{ABM-wikipedia}
Wikipedia.
\newblock {A}gent {B}ased {M}odels,
  \url{https://en.wikipedia.org/wiki/Agent-based_model}, 2020.

\bibitem[Zhou(2019)]{zhou19dis}
C.~Zhou.
\newblock \emph{Entropy, optimization and coding theory}.
\newblock PhD thesis, University of Notre Dame, Notre Dame, Indiana, 2019.

\bibitem[{\v Z}ivn{\' y} et~al.(2014){\v Z}ivn{\' y}, Werner, and Pr{\r u}{\v
  s}a]{14ZWP}
S.~{\v Z}ivn{\' y}, T.~Werner, and D.~Pr{\r u}{\v s}a.
\newblock The power of lp relaxation for map inference.
\newblock \emph{Advanced Structured Prediction, The MIT Press}, pages 19--42,
  December 2014.

\end{thebibliography}
\end{singlespace}

\appendix

\end{document}